\documentclass[journal, 11pt, onecolumn]{IEEEtran}

\usepackage{xparse}

\usepackage{amsmath}
\usepackage{amssymb}
\usepackage{amsthm}

\newtheorem{theorem}{Theorem}

\newtheorem{defn}{Definition}
\newtheorem{example}{Example}

\usepackage{graphicx}
\usepackage{booktabs}
\usepackage{array}
\usepackage[dvipsnames]{xcolor}
\usepackage{subcaption}
\usepackage{wrapfig}
\usepackage{multirow}
\usepackage{csquotes}

\usepackage{tikz}
\usepackage{pgfplots}
\usetikzlibrary{calc}
\usetikzlibrary{positioning}
\usetikzlibrary {shapes.geometric}
\usetikzlibrary{shapes.misc}
\usetikzlibrary{arrows}
\usetikzlibrary{arrows.meta}
\usepgflibrary{shapes.arrows}
\usepgflibrary{shapes.symbols}

\usepackage[noadjust]{cite}
\makeatletter
\let\NAT@parse\undefined
\makeatother
\usepackage[colorlinks,citecolor=black,linkcolor=black]{hyperref}
\usepackage[capitalise]{cleveref}
\crefname{equation}{}{}
\Crefname{equation}{}{}

\usepackage{censor}

\newcommand{\citet}[1]{{\color{red}TODO}~\cite{#1}}

\NewDocumentCommand{\legendline}{m}{{%
  \hspace*{-.2mm}\begin{tikzpicture}%
    \draw[line width=2pt, #1] (0.0cm,0.0cm) -- (0.3cm,0.2cm);
  \end{tikzpicture}\hspace*{-.2mm}%
}}

\newcommand\transp[1]{{%
    #1^{\mkern-1.5mu\mathsf{T}}
}}

\usepackage{xfrac}

\newcommand{\ourframework}{sNMODE}

\newcommand{\relu}[1]{\operatorname{ReLU}{(#1)}}

\newcommand{\Reals}{\mathbb{R}}
\newcommand{\RealsNonNeg}{\Reals_{\geq 0}}
\newcommand{\manif}{\mathcal{M}}

\newcommand{\unitquat}{S^3}
\newcommand{\spdtwo}{\mathcal{S}^{2}_{++}}
\newcommand{\distance}[1]{d_{#1}}
\newcommand{\tangentsp}[1]{\mathcal{T}_{#1}\mathcal{M}}
\newcommand{\tangentsptwo}[2]{\mathcal{T}_{#1}#2}
\newcommand{\innerprod}[3]{{\langle#2,#3\rangle}_{#1}}
\newcommand{\norm}[2]{{\|#2\|}_{#1}}
\newcommand{\liederiv}[2]{{\mathcal{L}_{#2}\;\!\!#1}}
\newcommand{\grad}[2]{\nabla_{#1}^{#2}}
\newcommand{\refpoint}{x_{r}}
\newcommand{\approxsol}{\tilde{\xi}}
\newcommand{\totangent}[1]{\operatorname{projt}_{#1}}

\definecolor{Colors-A}{RGB}{230,159,0}  
\definecolor{Colors-C}{RGB}{86,180,233}  
\definecolor{Colors-B}{RGB}{0,158,115}  
\definecolor{Colors-D}{RGB}{213,94,0}  
\definecolor{Colors-E}{RGB}{204,121,167}  
\definecolor{Colors-F}{RGB}{0,114,178}  
\definecolor{Colors-G}{RGB}{240,228,66}  
\definecolor{bettergreen}{RGB}{0,127,93}  
\definecolor{betterred}{RGB}{165,71,0}  
\definecolor{betterorange}{RGB}{178,121,0}  
\definecolor{pltblue}{RGB}{31,119,180}
\definecolor{pltorange}{RGB}{255,127,14}
\definecolor{brewergreen}{RGB}{35,139,69}
\definecolor{robotviolet}{RGB}{94,23,235}

\title{{\LARGE\bfseries%
  Stable Robot Motions on Manifolds: Learning\\ Lyapunov-Constrained Neural Manifold ODEs
}}

\author{%
    {David Boetius$^{1}$}, 
    {Abdelrahman Abdelnaby$^{1}$}, 
    {Ashok Kumar$^{2}$},\\
    {Stefan Leue$^{1}$},
    {Abdalla Swikir$^{2}$} and 
    {Fares J. Abu-Dakka$^{3}$}%
    \thanks{%
        $^{1}$David Boetius, Abdelrahman Abdelnaby, and Stefan Leue are with the Department of Computer and Information Sciences University of Konstanz, 78467 Konstanz, Germany {\tt\small \{david.boetius, abdelrahman.abdelnaby, stefan.leue\}@uni-konstanz.de}
    }
    \thanks{%
        $^{2}$Ashok Kumar and Abdalla Swikir are with the Department of Robotics, Mohamed bin Zayed University of Artificial Intelligence, Abu Dhabi, AE {\tt\small \{ashok.kumar,abdalla.swikir\}@mbzuai.ac.ae}
    }
    \thanks{%
        $^{3}$Fares J. Abu-Dakka is with the Mechanical Engineering Department, New York University Abu Dhabi, Abu Dhabi, AE {\tt\small fa2656@nyu.edu}
    }
}

\begin{document}
\maketitle

\thispagestyle{empty}
\pagestyle{empty}




\begin{abstract}
    Learning stable dynamical systems from data is crucial for safe and reliable robot motion planning and control. However, extending stability guarantees to trajectories defined on Riemannian manifolds poses significant challenges due to the manifold's geometric constraints. To address this, we propose a general framework for learning stable dynamical systems on Riemannian manifolds using neural ordinary differential equations. Our method guarantees stability by projecting the neural vector field evolving on the manifold so that it strictly satisfies the Lyapunov stability criterion, ensuring stability at every system state. By leveraging a flexible neural parameterisation for both the base vector field and the Lyapunov function, our framework can accurately represent complex trajectories while respecting manifold constraints by evolving solutions directly on the manifold.
    We provide an efficient training strategy for applying our framework and demonstrate its utility by solving Riemannian LASA datasets on the unit quaternion (\(\unitquat\)) and symmetric positive-definite matrix manifolds, as well as robotic motions evolving on~\(\Reals^3 \times \unitquat\).
    We demonstrate the performance, scalability, and practical applicability of our approach through extensive simulations and by learning robot motions in a real-world experiment.
\end{abstract}



\section{Introduction}
\begin{figure}[t]
    \centering
    \begin{tikzpicture}
        \node (Left) at (0,0) {};
        \node (Right) at ($(.55\linewidth,0)-(.25cm,0)$) {};
        

        \node[below right=-1.45cm and 0cm of Left.west] (demos) {\includegraphics[width=2.95cm]{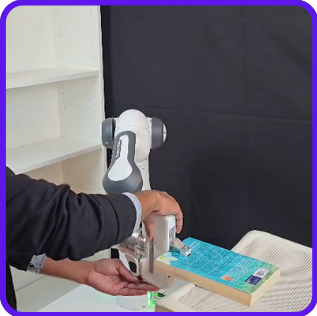}};
        \node[below left=-1.45cm and 0cm of Right.east] (robot) {\includegraphics[width=2.95cm]{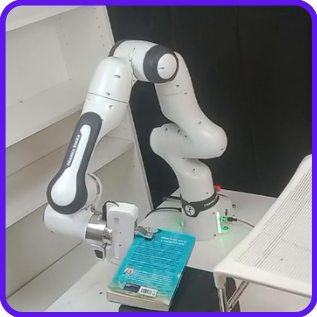}};
        
        \draw[Violet, line width=3pt,{Circle[width=3mm,length=3mm]}-] ($(Left) + (1.6cm,-1.2cm)$) -- ($(Left)!.5!(Right) - (0,0cm)$);
        \draw[Violet, line width=6pt,-{Triangle[width=4mm,length=4mm]}] ($(Left)!.5!(Right) - (0,0cm)$) to[bend left] ($(Right) - (2cm,-.85cm)$);
        \fill[Violet] ($(Left)!.5!(Right) - (0,.2cm)$) circle (1.6cm);
        
        \node[
            name=sNMODE, 
        ] at ($(Left)!.5!(Right) - (.04,.21cm)$) {
            \includegraphics[width=3.5cm]{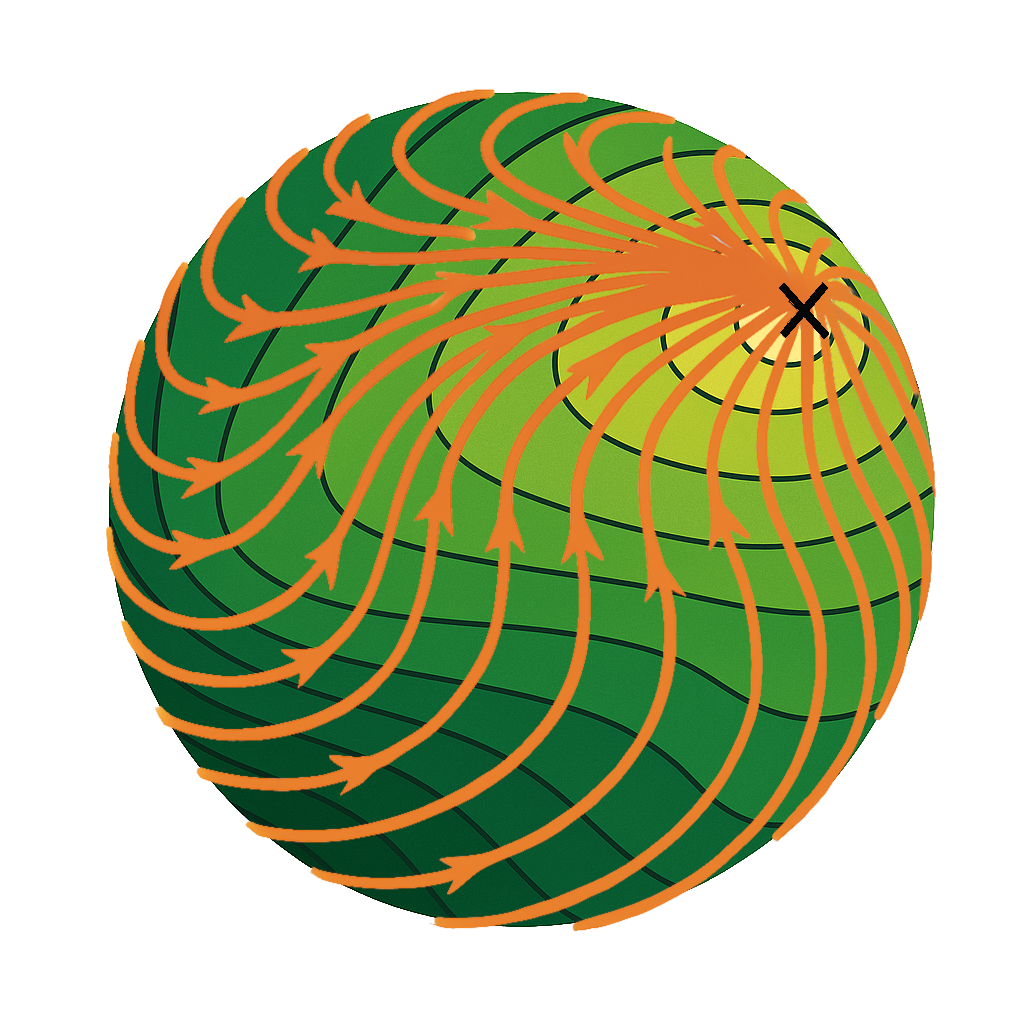}
        };
    \end{tikzpicture}
    \caption{%
        \textbf{Learning Stable Motions on Riemannian Manifolds using \ourframework{}}.
        A user demonstrates position-orientation motions (\emph{Left}), which lie on a Riemannian manifold (\emph{Centre});
        \ourframework{} learns a stable vector field~\legendline{pltorange} that enables the robot to autonomously perform the motion (\emph{Right}).
    }\label{fig:visual-abstract}
\end{figure}
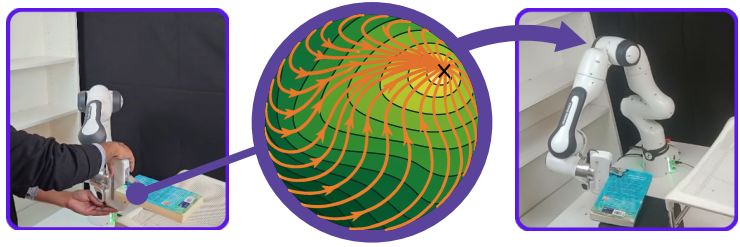

\noindent
The deployment of robotic systems in complex real-world scenarios demands models that are highly expressive and rigorously stable. Learning from demonstrations enables non-expert users to teach robots new skills by providing example trajectories~\cite{Schaal1999,Ravichandar2020Recent,WangSaverianoAbuDakka2022}. 
However, purely data-driven methods frequently lack formal stability guarantees required, in particular, in safety-critical domains. 

Robotic state variables, such as orientation, stiffness, or inertia, naturally reside on smooth non-Euclidean spaces that are inherently structured as Riemannian manifolds~\cite{WangSaverianoAbuDakka2022,AlesEtAl2014,Zeestraten2018,KoutrasDoulgeri2019,ZhangBeikMohammadiRozo2022,SaverianoAbuDakkaKyrki2023}.
To ensure reliable task execution, learned trajectories must accurately capture complex motion patterns on Riemannian manifolds while reliably steering the system toward a desired equilibrium state. Stable dynamical systems offer a principled solution by guaranteeing trajectory convergence and robustness to perturbations, significantly enhancing reliability in practice~\cite{IjspeertEtAl2013,LemmeEtAl2014,KhansariZadehBillard11,NeumannSteil2015,PerrinSchlehuberCaissier2016,BlocherSaverianoDongheui2017,DuanEtAl2019,NawazEtAl2024}. 

Recently, neural ordinary differential equations (NODEs)~\cite{NODEs} have emerged as an expressive framework for modelling continuous-time dynamical systems.
NODEs offer the flexibility to represent intricate trajectories effectively~\cite{LemmeEtAl2014,GruenbacherEtAl2021,WhiteEtAl2023,NawazEtAl2024,SochopoulosGiengerVijayakumar2024}, also in high-dimensional spaces~\cite{AuddyEtAl2023Scalable}. 
While stability of NODEs has been studied extensively within Euclidean spaces~\cite{ManekKolter2019,KangEtAl2021,LuoEtAl2025}, stability guarantees for NODEs on Riemannian manifolds remain an underexplored area.
Without stability guarantees, learned dynamics may drift, oscillate, or exhibit unpredictable behaviour, undermining both safety and performance.

In this paper, we introduce stable neural manifold ordinary differential equations (\ourframework{})~---~a novel stability framework that combines Lyapunov-based stability with NODEs evolving on arbitrary Riemannian manifolds.
Our approach jointly learns an expressive neural manifold vector field and a neural Lyapunov function. 
Projecting the vector field using a Lyapunov-based corrective term guarantees exponential stability towards a predefined equilibrium state~\cite{ManekKolter2019}.
We uncover and resolve a significant issue in several stable NODE frameworks~\cite{ManekKolter2019,SochopoulosGiengerVijayakumar2024,AuddyEtAl2023Scalable} that invalidates their stability guarantees and causes numerical instability in practice. 

For practically applying our framework, we devise an efficient training strategy that divides training \ourframework{}s into pretraining a base vector field on demonstrations, pretraining a Lyapunov function, and fine-tuning the composed system under the stability projection.
Our \ourframework{}s reliably replicate demonstration trajectories while providing strong stability guarantees.
They outperform existing approaches~\cite{ZhangBeikMohammadiRozo2022,SaverianoAbuDakkaKyrki2023} in terms of efficiency, respectively, performance and scalability.
In summary, our main contributions are
\begin{enumerate}
  \item proposing \ourframework{}, a novel manifold-aware Lyapunov projection method that ensures stability on Riemannian manifolds without imposing restrictive assumptions on the manifold or model structure,
  \item resolving a critical stability issue in previous frameworks~\cite{ManekKolter2019,SochopoulosGiengerVijayakumar2024,AuddyEtAl2023Scalable} that invalidates theoretical stability guarantees and causes numerical instability in practice,
  \item providing an efficient training strategy for learning stable robot motions from demonstrations,
  \item addressing a fundamental limitation of the Riemannian LASA dataset~\cite{SaverianoAbuDakkaKyrki2023} regarding data realism, and
  \item performing an extensive experimental evaluation on several Riemannian manifolds, both in simulation and on a real-world robot.
\end{enumerate}


\section{Related Work}\label{sec:related-work}
\newcommand{\ck}{\textcolor{bettergreen}{\(\pmb{\checkmark}\)}}
\newcommand{\no}{\textcolor{betterred}{\(\pmb{\times}\)}}
\newcommand{\na}{--}
\begin{table}[b]
    \centering
    \caption{\textbf{Comparison with Existing LfD Frameworks}.}
    \label{tab:related-work-comparison}
    \begin{tabular}{>{\centering}p{.13\columnwidth} @{} >{\centering}p{.18\columnwidth} @{} >{\centering}p{.18\columnwidth} @{} >{\centering}p{.18\columnwidth} @{} >{\centering\arraybackslash}p{.12\columnwidth}}
        \ \newline\textbf{Reference} & \textbf{Riemannian}\newline\textbf{Manifold} & \ \textbf{Stability}\newline\textbf{Guarantee} & \ \textbf{Learned}\newline\textbf{Vector Field} & \textbf{Neural}\newline\textbf{ODE} \\ \midrule
        
        \cite{AuddyEtAl2023Continual}
        & \no & \no & \ck &\ck \\
        
        \cite{SochopoulosGiengerVijayakumar2024,AuddyEtAl2023Scalable}
        & \no & \(\textcolor{betterorange}{\pmb{\checkmark}}^{\;\!\!\ast}\) & \ck &\ck \\
        
        \cite{LemmeEtAl2014,NawazEtAl2024}
        & \no & \ck & \ck & \ck \\

        \cite{KhansariZadehBillard11}
        & \no & \ck & \ck & \no \\
        
        \cite{WangSaverianoAbuDakka2022}
        & \ck & \no & \ck & \ck \\
        
        \cite{SaverianoAbuDakkaKyrki2023}
        & \ck & \ck & \no &\no \\
        
        \cite{ZhangBeikMohammadiRozo2022}
        & \ck & \(\textcolor{betterorange}{\pmb{\checkmark}}^{\;\!\!\dagger}\)  & \no & \ck \\
        
        \textbf{Ours}
        & \ck & \ck &\ck & \ck \\
    \end{tabular}
    
    {\footnotesize \({}^\ast\)Affected by the equilibrium point issue that we address in this paper.}\\
    {\footnotesize \({}^\dagger\)Unlike the other approaches, the equilibrium point is not user-defined.}
\end{table}
In this section, we review existing works related to stable neural ODEs (NODEs) on Riemannian manifolds. \Cref{tab:related-work-comparison} provides a concise comparison of our \ourframework{} framework to existing learning from demonstrations (LfD) frameworks.
We first discuss approaches in Euclidean space.

\paragraph{NODEs and Stability}
NODEs~\cite{NODEs} are a powerful paradigm for modelling complex continuous-time dynamical systems, from image processing~\cite{NODEs,KangEtAl2021,LuoEtAl2025} to learning robotic skills~\cite{LemmeEtAl2014,WhiteEtAl2023,NawazEtAl2024,SochopoulosGiengerVijayakumar2024,AuddyEtAl2023Continual,AuddyEtAl2023Scalable}.
Stability in NODEs is addressed using Lyapunov functions~\cite{ManekKolter2019,AuddyEtAl2023Scalable,KangEtAl2021,SochopoulosGiengerVijayakumar2024,LuoEtAl2025} and contraction metrics~\cite{KhansariZadehBillard11, NawazEtAl2024}. 

However, these stability frameworks are fundamentally limited to Euclidean spaces and cannot handle data evolving on more general Riemannian manifolds. Additionally, we identify a critical equilibrium consistency issue in~\cite{ManekKolter2019} that invalidates theoretical stability guarantees and causes numerical instability in practice. This issue also affects subsequent works~\cite{AuddyEtAl2023Scalable,SochopoulosGiengerVijayakumar2024} that build upon~\cite{ManekKolter2019}. We address this critical issue in \cref{sec:main-stability-issue}.
%

\paragraph{NODEs on Riemannian Manifolds}
Neural Manifold ODEs (NMODEs)~\cite{DynamicChartMethod} provide an expressive framework for learning dynamical systems while respecting manifold geometry. 
While~\cite{DynamicChartMethod} employs dynamic coordinate charts to solve NMODEs,~\cite{WangSaverianoAbuDakka2022} projects the manifold data into a single tangent space.
Both~\cite{DynamicChartMethod} and~\cite{WangSaverianoAbuDakka2022} lack stability guarantees.

\paragraph{Stable Dynamical Systems on Riemannian Manifolds} 
Learning stable dynamical systems on Riemannian manifolds combines differential geometry, control theory, and machine learning challenges. Recent approaches~\cite{ZhangBeikMohammadiRozo2022,SaverianoAbuDakkaKyrki2023} focus on learning diffeomorphic mappings to transform stable canonical dynamics into complex behaviours using NMODEs and Gaussian mixture models (GMMs).

RSDS~\cite{ZhangBeikMohammadiRozo2022} uses NMODEs to construct diffeomorphisms for Riemannian stable dynamical systems, providing theoretical stability guarantees through pullback operations. However, this method is costly, as it solves an NMODE at each time step, hindering real-time applications. 
Additionally, a fundamental issue in~\cite{ZhangBeikMohammadiRozo2022} is that the equilibrium of the learned diffeomorphism is assumed to automatically correspond to the desired equilibrium point, without enforcing this constraint during training. While~\cite{ZhangBeikMohammadiRozo2022} ensures convergence to \emph{some} equilibrium point, users have no control over the location of this equilibrium point.

In contrast, SDS-RM~\cite{SaverianoAbuDakkaKyrki2023} employs GMMs to effectively learn diffeomorphisms. Compared to RSDS, SDS-RM ensures computational efficiency while guaranteeing convergence to a user-defined equilibrium point. Nevertheless, GMMs struggle with scalability in high-dimensional manifolds~\cite{abu-dakka2018force} and oftentimes fail to capture the complex nonlinear dynamics that neural networks handle well.

Alternative approaches~\cite{ravichandar2019learning,mukadam2020riemannian} explore contraction theory on manifolds~\cite{ravichandar2019learning} and Riemannian motion policies~\cite{mukadam2020riemannian}, but focus on tracking a reference trajectory rather than learning stable dynamics from demonstrations.

\paragraph{Positioning Our Approach}
Our work addresses key limitations in current methods by merging the expressiveness of NODEs with stability guarantees on Riemannian manifolds.
Unlike diffeomorphism-based methods~\cite{ZhangBeikMohammadiRozo2022,SaverianoAbuDakkaKyrki2023}, we learn stable vector fields on manifolds directly, avoiding expensive inverse mappings. 
Our approach retains the computational benefits of NODEs while offering the geometric awareness crucial for robotic applications involving orientation, impedance, or other manifold-valued quantities.

%


\section{Background}\label{sec:background}
This section presents the mathematical foundations of Riemannian manifolds, Lyapunov stability, and numerical methods for solving ODEs on Riemannian manifolds.

\subsection{Riemannian Manifolds}\label{sec:riemannian-manifolds}
An~\(n\)-dimensional \emph{manifold}~\(\manif\) is a topological space that locally resembles Euclidean space in the neighbourhood of each point~\(x \in \manif\).
Each manifold is embedded in an Euclidean \emph{ambient space}~\(\Reals^m \supseteq \manif\),~\(m \geq n\).
Each manifold is defined by a collection of charts~\({(\varphi_x: U_x \to U_x')}_{x \in A}\) with index set~\(A\), where each~\(\varphi_x\) is a homomorphism between the open sets~\(U_x \subseteq \manif\) and~\(U_x' \subseteq \Reals^n\).
A manifold is \emph{smooth} if each \emph{transition map}~\(\varphi_x \circ \varphi_y^{-1}\) is smooth for overlapping charts~\(x, y \in A\).
%
Each point~\(x \in \manif\) on a smooth manifold corresponds to an~\(n\)-dimensional \emph{tangent space}~\(\tangentsp{x} \subseteq \Reals^m\), which contains the velocities of curves passing through~\(x\). 
The \emph{tangent bundle}~\(\tangentsp{}= \bigcup_{x \in \manif} \tangentsp{x}\) is the disjoint union of all tangent spaces.
The orthogonal projection~\(\totangent{x}: \Reals^m \to \tangentsp{x}\) maps vectors from the ambient space to the tangent space at~\(x\).
%

A \emph{Riemannian manifold}~\((\manif, g)\) is a smooth manifold equipped with a \emph{Riemannian metric}~\(g\), which assigns to each tangent space~\(\tangentsp{x}\) a positive-definite \emph{inner product}~\(\innerprod{x}{\cdot}{\cdot}: \tangentsp{x} \times \tangentsp{x} \to \Reals\) that varies smoothly with~\(x\).
The induced norm is~\(\norm{x}{u} = \sqrt{\innerprod{x}{u}{u}}\).
The Riemannian metric enables the definition of geometric concepts. 
In particular, the \emph{length} of a piecewise differentiable curve~\(\xi: [a, b] \to \manif\) is~\(\ell(\xi) = \int_a^b \norm{\xi(t)}{\dot{\xi}(t)}dt\), where~\(\dot{\xi}(t) = \frac{d \xi(s)}{d s}|_{s=t}\) is the velocity of~\(\xi\).
The \emph{geodesic}~\(\gamma_{x,y}: [0, 1] \to \manif\) is the shortest curve with~\(\gamma_{x,y}(0) = x\),~\(\gamma_{x,y}(1) = y\), and constant velocity~\(\norm{\gamma(t)}{\dot{\gamma}(t)}, \forall t \in [0, 1]\).
The Riemannian \emph{distance} between~\(x,y \in \manif\) is~\(\distance{\manif}(x, y) = \ell(\gamma_{x,y})\).
The \emph{exponential map}~\(\exp_x: \tangentsp{x} \to \manif\) is defined as~\(\exp_x(u) = \gamma_{x}^u(1)\), where~\(\gamma_x^u\) is the unique geodesic satisfying~\(\gamma_x^u(0) = x\) and~\(\dot{\gamma}_x^u(0) = u\). Its inverse, the \emph{logarithmic map}~\(\log_x: \manif \to \tangentsp{x}\), exists locally and satisfies~\(\exp_x(\log_x(y)) = y\) for~\(y\) sufficiently close to~\(x\).

\paragraph{Differential Equations on Manifolds} 
A \emph{manifold ordinary differential equation (MODE)} is specified by a vector field~\(f: \manif \to \tangentsp{}\) such that~\(f(x) \in \tangentsp{x}\) for all~\(x \in \manif\).
The solution of a MODE is the curve~\(\xi : [0, T] \to \manif\) that satisfies the initial value problem
\begin{equation}
    \frac{d\xi(t)}{dt} = f(\xi(t)) \in \tangentsp{\xi(t)}, \quad \xi(0) = x_0, \quad \xi(t) \in \manif.\label{eqn:mode}
\end{equation}
If~\(f\) is parameterised by a neural network, \cref{eqn:mode} is a \emph{neural manifold ordinary differential equation (NMODE)}.

\paragraph{Lie Derivative} 
The \emph{Lie derivative} of a smooth function~\(F: \manif \to \Reals\) with respect to the vector field~\(f: \manif \to \tangentsp{}\) provides a directional derivative
\begin{equation*}    
  \liederiv{F}{f}(x) = \frac{d}{dt} \left. F(\xi_x(t)) \right\rvert_{t=0} = \innerprod{x}{\grad{x}{\manif} F(x)}{f(x)}
\end{equation*}
where~\(\xi_x\) is the solution of \cref{eqn:mode} with~\(\xi_x(0) = x\) and \(\grad{x}{\manif} F(x) \in \tangentsp{x}\) is the Riemannian gradient of~\(F\) at~\(x\).
Below, we provide an example of a Riemannian manifold: the unit quaternion manifold.
We refer to~\cite{SaverianoAbuDakkaKyrki2023} and~\cite{DynamicChartMethod} for further examples of Riemannian manifolds.

\begin{example}
    The set of unit quaternions~\(\unitquat = \{x \in \Reals^4 \mid \|x\|_2 = 1\}\) equipped with the Euclidean inner product~\(\innerprod{z}{x}{y} = \transp{x}y\) forms a three-dimensional Riemannian manifold embedded in~\(\Reals^4\). The exponential map~\(\exp_x: \tangentsptwo{x}{\unitquat} \to \unitquat\) is given by~\(\exp_x(u) = \cos(\|u\|)x + \sin(\|u\|)\frac{u}{\|u\|}\) for~\(u \neq 0\) and~\(\exp_x(0) = x\). The Lie derivative on~\(\unitquat\) is~\(\liederiv{F}{f}(x) = \transp{\grad{x}{\unitquat} F(x)} f(x)\), where~\(\grad{x}{\unitquat}F(x) = \totangent{x}(\grad{x}{\Reals^4}F(x)) = \grad{x}{\Reals^4}F(x) - x\innerprod{x}{\grad{x}{\Reals^4}F(x)}{x}\).
\end{example}

\subsection{Stability}\label{sec:stability-background}
In this paper, we are primarily concerned with exponential stability on the Riemannian manifold~\(\manif\).
Since some manifolds, such as~\(\unitquat\), do not permit global stability due to the Pointcar{\'e}-Hopf Theorem~\cite{milnor1997topology}, we consider quasi-global stability, which allows additional unstable equilibrium points.

\begin{defn}[Quasi-Global Stability]\label{defn:stability}
    Let~\(f: \manif \to \tangentsp{}\) be a vector field.
    A point~\(x_e\) is an \emph{equilibrium point} if~\(f(x_e) = 0\).
    It is a \emph{stable equilibrium point} if there exist~\(\alpha \geq 0, \beta > 0\) such that~\(\distance{\manif}(\xi(t), x_e) \leq \beta e^{-\alpha t}\distance{\manif}(x_0, x_e)\) for every~\(x_0 \in \manif \setminus E\), where~\(\xi\) is the solution of \cref{eqn:mode} and~\(E \subset \manif\) is a finite set.
    If~\(\alpha > 0\),~\(x_e\) is \emph{\(\alpha\)-exponentially stable}.
\end{defn}
%

\begin{theorem}[Lyapunov Function]\label{thm:lyapunov-fn}
    Let~\(f: \manif \to \tangentsp{}\) with equilibrium point~\(x_e\).
    Let~\(V: \manif \to \RealsNonNeg\) be a continuously differentiable function with~\(V(x_e) = 0\) and~\(V(x) > 0\) for~\(x \in \manif \setminus E\), where~\(E \subset \manif\) is finite.
    If there exist~\(\alpha \geq 0, c_1, c_2, p > 0\), such that~\(\liederiv{V}{f}(x) \leq -\alpha V(x)\) and~\(c_1 \distance{\manif}(x, x_e)^p \leq V(x) \leq c_2 \distance{\manif}(x, x_e)^p\) for all~\(x \in \manif\),~\(x_e\) is a stable equilibrium point.
    If~\(\alpha > 0\),~\(x_e\) is \(\alpha\)-exponentially stable.
\end{theorem}
We refer to~\cite{LyapunovFn} for a proof of \cref{thm:lyapunov-fn}.

\subsection{Numerical Solutions of MODEs}\label{sec:solve-mode}
For training NMODEs, we require a method for numerically solving MODEs.
A baseline approach for solving the MODE \cref{eqn:mode} is the \emph{tangent space method (TS)} ~\cite{WangSaverianoAbuDakka2022}.
TS constructs an approximate numerical solution~\(\approxsol_{0:N}\) of \cref{eqn:mode} by computing a numerical solution~\(\tilde{\zeta}_{0:N}\) of the ODE
\begin{equation*}
    \frac{d \zeta(t)}{dt} = (\totangent{\refpoint} \circ f \circ \exp_{\refpoint})(\zeta(t)), \quad
    \zeta(0) = \log_{\refpoint}(x_0),
\end{equation*}
which evolves in the tangent space of a reference point~\(\refpoint \in \manif\).
The numerical solution~\(\tilde{\zeta}\) is computed using a standard technique for solving ODEs, such as the Euler method, or a Runge-Kutta method.
By projecting~\(\tilde{\zeta}\) to the manifold, we obtain the numerical solution of~\cref{eqn:mode} as~\(\approxsol_i = \exp_{\refpoint}(\tilde{\zeta}_i) \in \manif\) for each~\(i \in \{0, \ldots, N\}\).
In our setting, the equilibrium point~\(x_e\) is a natural choice for the reference point~\(\refpoint\).

By evolving in only one tangent space, TS can introduce significant approximation error. 
The \emph{exp-step method (Exp)} reduces this approximation error by switching the reference point frequently.
The steps of the approximate solution~\(\approxsol_{0:N}\) computed by Exp are~\(\approxsol_0 = x_0\) and~\(\approxsol_{i+1} = \exp_{\approxsol_{i}}(\hat{\zeta}_i')\) for~\(i \in \{0, \ldots, N-1\}\), where~\(\tilde{\zeta}_i'\) is the endpoint of a numerical solution of the ODE
\begin{equation*}
    \frac{d \zeta_i(t)}{dt} = (\totangent{\approxsol_i} \circ f \circ \exp_{\approxsol_i})(\zeta_i(t)), \quad
    \zeta_i(0) = 0,
\end{equation*}
over the the time interval~\(t \in [(i-1)T/N, iT/N]\).

Finally, the \emph{dynamic-chart method (DC)}~\cite{DynamicChartMethod} leverages charts instead of tangent spaces to transfer Euclidean ODE solutions to the manifold. 
Let~\(\varphi_0: U_0 \to V_0\) be a chart of~\(\manif\) with~\(x_0 \in U_0\).
DC solves an ODE in~\(V_0\) until the solution approaches the boundary of~\(V_0\).
Once this is the case, the method switches to an adjacent chart~\(\varphi_1: U_1 \to V_1\) and solves another ODE, switching the chart again when the solution of the ODE approaches the boundary of~\(V_1\).
The ODE that is solved for the chart~\(\varphi_i : U_i \to V_i\) is
\begin{equation*}
    \frac{d\zeta_i(t)}{dt} = (D_{\varphi_i^{-1}(\zeta_i(t))}^{\manif}\varphi_i \circ f \circ \varphi_i^{-1})(\zeta_i(t)), \quad 
    \zeta_i(t) \in V_i,
\end{equation*}
where~\(D_{x}^{\manif} \varphi_i: \tangentsp{x} \to \Reals^n\) is the pushforward of~\(\varphi_i\) with respect to~\(x\) and the initial point~\(\zeta_i(0)\) is the end point of the trajectory in the previous chart~\(\varphi_{i-1}\), respectively,~\(x_0\) for~\(i=0\).
The steps of the overall numerical solution are~\(\approxsol_0 = x_0\) and~\(\approxsol_i = \varphi_i^{-1}(\tilde{\zeta}_i'')\) for~\(i \in \{1, \ldots, N\}\) where~\(\tilde{\zeta}_i''\) is the endpoint of the trajectory in the chart~\(\varphi_{i-1}\).
We refer to~\cite{DynamicChartMethod} for more details on DC.


\section{Stable Neural Manifold ODEs}\label{sec:main}
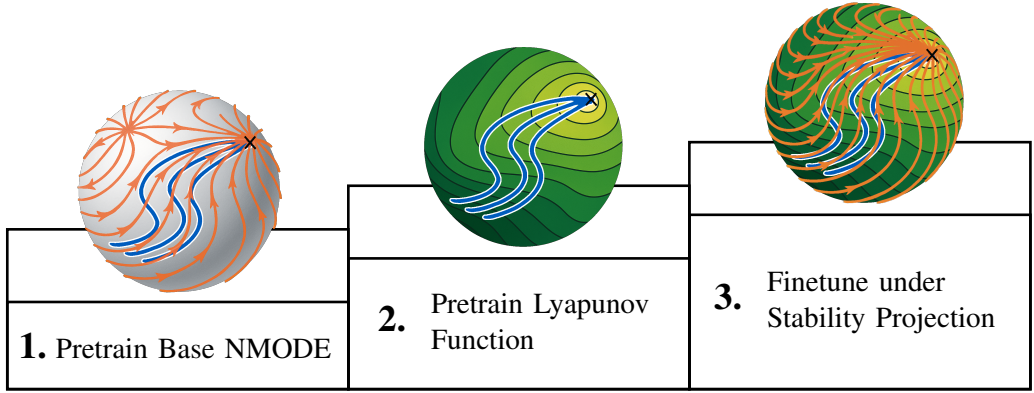
\begin{figure}
    \centering
    \begin{tikzpicture}[
        stepshape/.style={
            shape=rectangle, 
            draw=black, fill=white, 
            line width=1.75pt, 
            minimum height=3cm, 
            minimum width=1cm,
            align=center,
        },
        steparrow/.style={
            single arrow, shape border rotate=0,
            draw=black, thin,
            minimum width=.5cm,
            single arrow tip angle=90,
            single arrow head extend=.125cm,
        },
    ]
        \node (Left) at (0,0) {};
        \node (Right) at ($(.5\linewidth,0)-(.25cm,0)$) {};
        
        \node[
            draw=black, rectangle, line width=1pt, 
            minimum width=.25\linewidth, minimum height=5.5em, 
            align=left, anchor=south west,
        ] (box1) at (0, 0) {};
        \node[
            draw=black, rectangle, line width=1pt, 
            minimum width=.25\linewidth, minimum height=7em, 
            align=left, anchor=south west,
        ] (box2) at (.25\linewidth, 0) {};
        \node[
            draw=black, rectangle, line width=1pt, 
            minimum width=.25\linewidth, minimum height=8.5em, 
            align=left, anchor=south west,
        ] (box3) at (.5\linewidth, 0) {};

        \node[
            draw=black, rectangle, line width=1pt, 
            minimum width=.25\linewidth, minimum height=3em, 
            align=left, anchor=south west,
        ] (stage1) at (0, 0) {%
            \Large\textbf{1.}
            \normalsize
            Pretrain Base NMODE\!
        };
        \node[
            draw=black, rectangle, line width=1pt, 
            minimum width=.25\linewidth, minimum height=4.5em, 
            align=left, anchor=south west,
        ] (stage2) at (.25\linewidth, 0) {%
            \Large\textbf{2.}
            \normalsize
            \begin{tabular}{l@{}}
                Pretrain
                Lyapunov\\
                Function
            \end{tabular}
        };
        \node[
            draw=black, rectangle, line width=1pt, 
            minimum width=.25\linewidth, minimum height=6em, 
            align=left, anchor=south west,
        ] (stage3) at (.5\linewidth, 0) {%
            \Large\textbf{3.}
            \normalsize
            \begin{tabular}{l@{}}
                Finetune
                under\\
                Stability
                Projection
            \end{tabular}
        };
        
        \node[
            above right=-.3cm and 0.55cm of stage1.north west,
            name=img1, 
        ]{%
            \includegraphics[width=3.25cm]{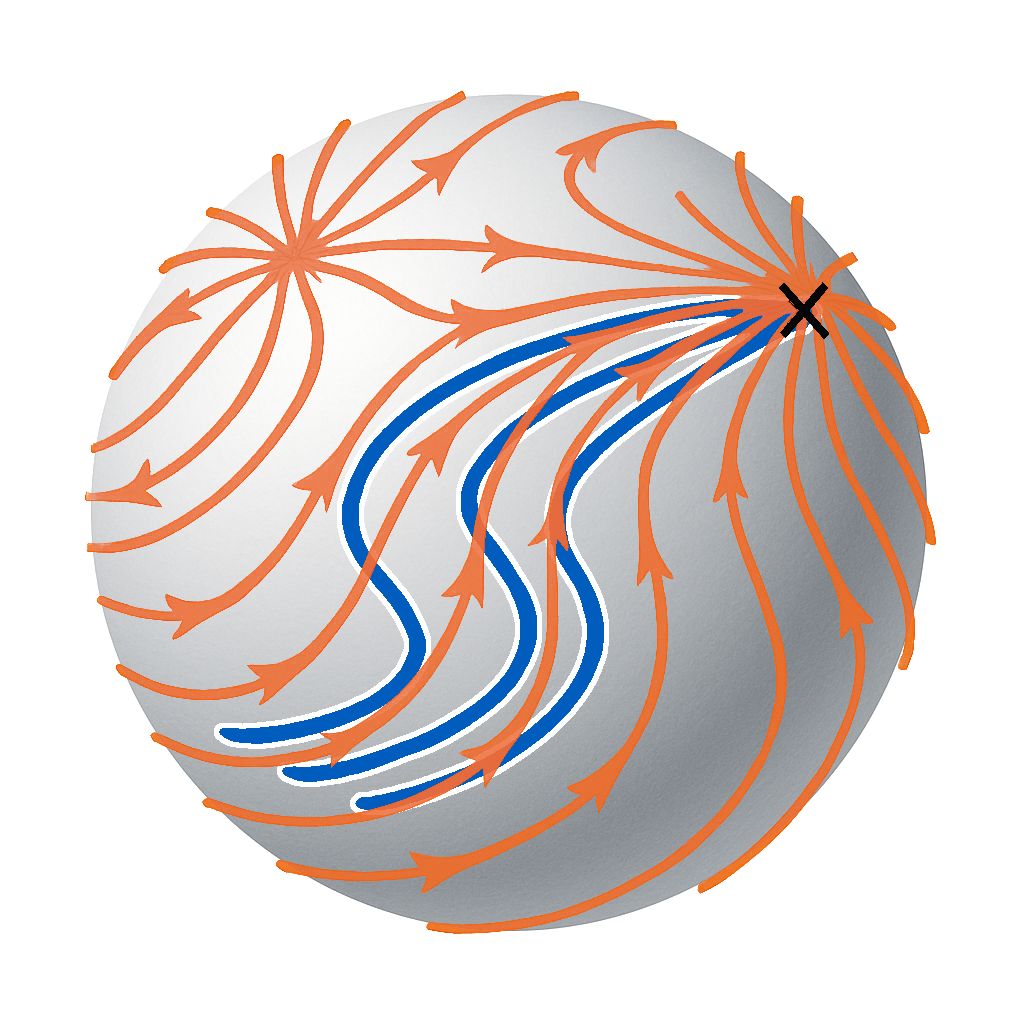}
        };
        \node[
            above right=-.3cm and 0.55cm of stage2.north west,
            name=img2, 
        ]{%
            \includegraphics[width=3.25cm]{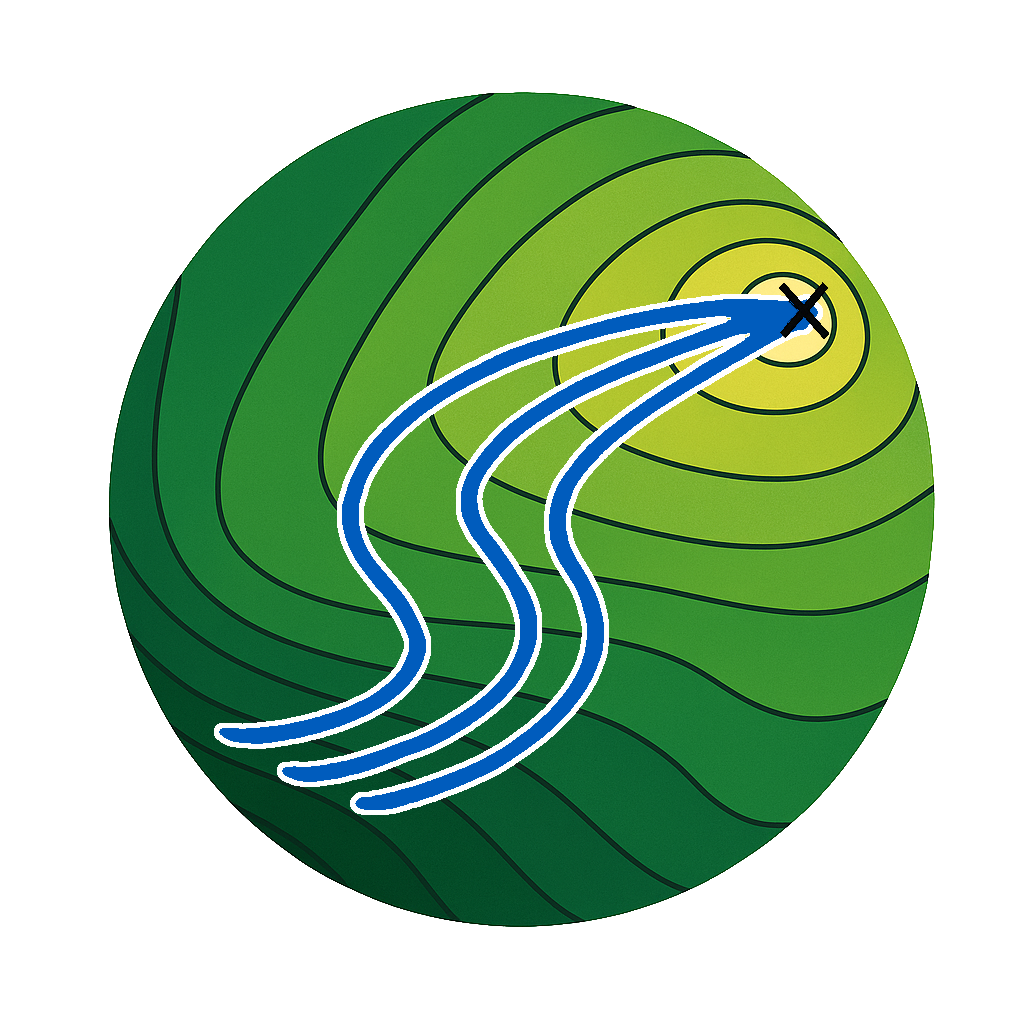}
        };
        \node[
            above right=-.3cm and 0.55cm of stage3.north west,
            name=img3, 
        ]{%
            \includegraphics[width=3.25cm]{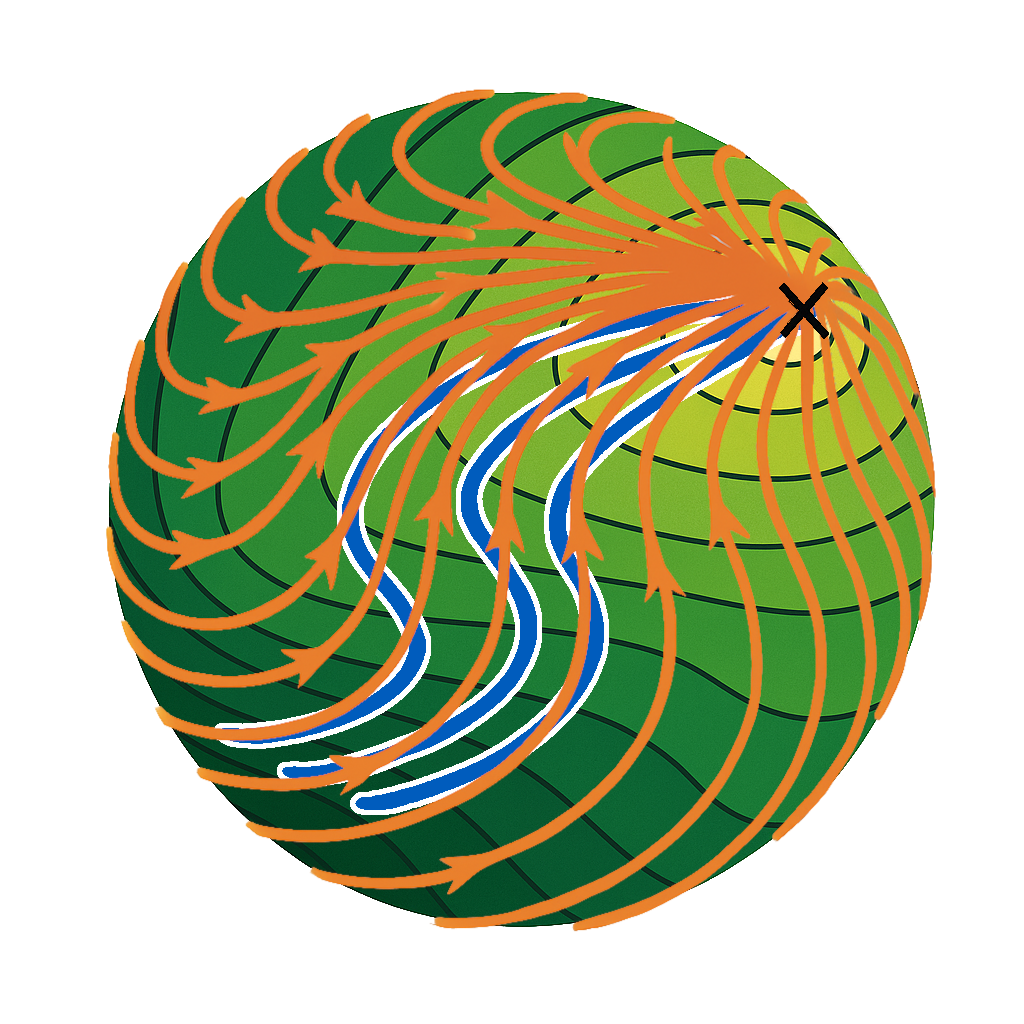}
        };

    \end{tikzpicture}
    \caption{
        \textbf{Our \ourframework{} Three-Stage Training Strategy}.
        The NMODEs \legendline{pltorange} and the Lyapunov function \legendline{brewergreen} are trained on the human demonstrations \legendline{pltblue}.
    }\label{fig:framework}
\end{figure}

In this section, we introduce~\ourframework{}, a novel expressive framework for learning quasi-globally exponentially stable neural ODEs defined on Riemannian manifolds. 
Our approach leverages the Lyapunov stability principles and generalises the successful stability projection proposed in~\cite{ManekKolter2019} from Euclidean space to arbitrary Riemannian manifolds. 
However, we find that~\cite{ManekKolter2019} has a critical issue that invalidates the theoretical stability guarantees and leads to numerical instability in practice.
We address this issue in our NMODE architecture, introduced in \cref{sec:nmode-arch}.
\Cref{sec:nmode-stability} contains our stability proof and \cref{sec:main-stability-issue} discusses the implications of the stability issue in~\cite{ManekKolter2019} in detail.
Before progressing to our experiments, we introduce our efficient training strategy in \cref{sec:nmode-train}, also visualised in \cref{fig:framework}.
%

\subsection{NMODE Architecture}\label{sec:nmode-arch}
Let~\(x_e\) be the desired equilibrium point of the \ourframework{} to train.
Our approach is based on a neural base vector field~\(g: \manif \to \tangentsp{}\) and a positive definite neural Lyapunov function~\(V: \Reals^m \to \Reals\).
For obtaining a well-defined NMODE, we also require that~\(\grad{x}{\manif} V(x) \neq 0\) for every~\(x \in \manif \setminus \{x_e\}\).
This is ensured by combining a convex neural network~\cite{ICNN} with an invertible feature transformation~\cite{ManekKolter2019}.

\paragraph{Base Vector Field}
We first describe the architecture of our base vector field.
Let~\(h: \Reals^m \to \Reals^m\) be any smooth neural network.
Our base vector field~\(g: \manif \to \tangentsp{}\) is
\begin{equation}
    g(x) = \totangent{x}(h(x) - h(x_e)).\label{eqn:base-deriv}
\end{equation}
This definition ensures~\(g(x) \in \tangentsp{x}\) and~\(g(x_e) = 0\), so that~\(x_e\) is an equilibrium point of~\(g\).
When not enforcing~\(g(x_e) = 0\),~\(g\) may learn an incorrect equilibrium point or may not even possess an equilibrium point.
The stability issue in~\cite{ManekKolter2019} is failing to enforce~\(g(x_e) = 0\).
The implications of this are discussed in detail in \cref{sec:main-stability-issue}.

\paragraph{Lyapunov Function}
We define the candidate Lyapunov function~\(V: \manif \to \RealsNonNeg\), similarly to~\cite{ManekKolter2019}, as
\begin{equation}
   V(x) = \sigma(C(F(x)) - C(F(x_e))) + \varepsilon \distance{\manif}(x, x_e)^2,\label{eqn:neural-lyafn}
\end{equation}
where~\(F: \Reals^m \to \Reals^{m'}\) is a smooth invertible Lipschitz neural network~\cite{Lipschitz1,Lipschitz2},~\(C: \Reals^{m'} \to \Reals\) is a smooth convex neural network,~\(\varepsilon > 0\), and~\(\sigma\) is the smoothed ReLU function~\cite{ManekKolter2019}.
Eq.~\eqref{eqn:neural-lyafn} ensures~\(V(x_e) = 0\) and~\(V(x) > 0\) for~\(x \in \Reals^m \setminus \{x_e\}\), while the invertibility of~\(F\) and the convexity of~\(C\) ensure that~\(\grad{x}{\manif} V(x) \neq 0\) for~\(x \in \Reals^m \setminus \{x_e\}\)~\cite{ManekKolter2019}.
Concretely, we use~\(F(x) = \transp{[H(x), x]}\), where~\(H: \Reals^m \to \Reals^{k'}\) is a smooth neural network and invertibility is ensured by carrying over the input~\(x\) to the output.
Following \cite{ManekKolter2019}, we use an input convex neural network (ICNN)~\cite{ICNN} for~\(C\).

\paragraph{Stable Vector Field}
Given the base vector field~\(g\), the Lyapunov function~\(V\), and the desired exponential stability rate~\(\alpha \geq 0\), we can now define the vector field~\(f: \manif \to \tangentsp{}\) of our \ourframework{} as
\begin{align}
    f(x) = g(x) - \grad{x}{\manif}V(x)\frac{\relu{\liederiv{V}{g}(x) + \alpha V(x)}}{\norm{x}{\grad{x}{\manif} V(x)}^2},\label{eqn:stable-nmode}
\end{align}
where we use the convention~\(0 \cdot \frac{x}{0} = 0\) when~\(\grad{x}{\manif}V(x) = 0\).
For isolated points,~\(\grad{x}{\manif} V(x)\) can be undefined since~\(\grad{x}{\manif} \distance{\manif}(x, x_e)\) can be undefined on manifolds that only permit quasi-global stability, instead of global stability.
In this case, we let~\(f(x) = 0\).

\subsection{Stability Guarantee}\label{sec:nmode-stability}
In this section, we prove that the \ourframework{}~\eqref{eqn:stable-nmode} is stable.

\begin{theorem}\label{thm:nmode-stable}
    Let~\(x_e\),~\(f\), and~\(\alpha\) be as in \cref{sec:nmode-arch}.
    The point~\(x_e\) is a stable equilibrium point of \cref{eqn:stable-nmode}.
    If~\(\alpha > 0\),~\(x_e\) is \(\alpha\)-exponentially stable.
\end{theorem}

\begin{proof}
  Let~\(x_e\),~\(f\),~\(g\),~\(V\), and~\(\alpha\) be as in \cref{sec:nmode-arch}.
  Since~\(V(x_e) = 0\) and~\(g(x_e) = \totangent{x_e}(0) = 0\),~\(x_e\) is an equilibrium point of~\(f\).
  We now show that~\(V\) satisfies \cref{thm:lyapunov-fn}.
  \Cref{thm:nmode-stable} then follow from \cref{thm:lyapunov-fn} immediately.
  First, we have~\(V(x_e) = \sigma(C(F(x_e)) - C(F(x_e)) + \varepsilon\distance{\manif}(x_e, x_e)^2 = \sigma(0) + \varepsilon\distance{\manif}(x_e, x_e)^2 = 0\).
  Second,~\(V(x) > \varepsilon\distance{\manif}(x, x_e)^2 > 0\) for~\(x \in \manif \setminus (\{x_e\} \cup E)\), where~\(E\) are the finitely many points there~\(\grad{x}{\manif} V(x)\) is undefined.
  This also shows that~\(V(x) \geq c_1\distance{\manif}(x, x_e)^2\) for~\(c_1 = \varepsilon > 0\).
  The upper bound~\(V(x) \leq c_2 \distance{\manif}(x, x_e)^2\) follows from~\(F\),~\(\sigma\),and~\(C\) being Lipschitz~\cite{ManekKolter2019}.
  It remains to show that~\(\liederiv{V}{f}(x) \leq - \alpha V(x)\)
  \begin{IEEEeqnarray*}{Cl}
      & \liederiv{V}{f}(x) \\
      =& \innerprod{x}{\grad{x}{\manif}V(x)}{f(x)} \\
      & {}- \innerprod{x}{\grad{x}{\manif}V(x)}{\grad{x}{\manif}V(x)}\frac{\relu{\liederiv{V}{g}(x) + \alpha V(x)}}{\norm{x}{\grad{x}{\manif} V(x)}^2} \\
      =& \liederiv{V}{g}(x) - \relu{\liederiv{V}{g}(x) + \alpha V(x)} \leq -\alpha V(x).
  \end{IEEEeqnarray*}
  The last equation is established by differentiating two cases based on whether~\(\relu{\liederiv{V}{g}(x) + \alpha V(x)}\) is positive or non-positive.
  This establishes \cref{thm:nmode-stable}.
\end{proof}

\subsection{Stability Issue in~\texorpdfstring{\cite{ManekKolter2019}}{Manek \& Kolter (2019)}}\label{sec:main-stability-issue}
The issue in \cite{ManekKolter2019} is that~\cite{ManekKolter2019} does not ensure the desired equilibrium point~\(x_e\) is actually an equilibrium point of the NODE~\(f\).
This issue has carried over to later works~\cite{AuddyEtAl2023Scalable,SochopoulosGiengerVijayakumar2024}, which build upon~\cite{ManekKolter2019}.
What may appear as a minor technical detail has, in fact, significant theoretical and practical implications that we discuss in this section.

The framework of~\cite{ManekKolter2019} is a special case of our framework when the Riemannian manifold is Euclidean space, except for one difference: In~\cite{ManekKolter2019},~\(g\) from \cref{eqn:base-deriv} is unconstrained, such that we may have~\(g(x_e) \neq 0\).
In consequence, the stabilised vector field~\(f\) may also have~\(f(x_e) \neq 0\) in \cite{ManekKolter2019}.
Since~\(f(x_e) = 0\) is a prerequisite for stability, strictly speaking,~\cite{ManekKolter2019} does not offer stability guarantees. 
Effectively,~\cite{ManekKolter2019} enforces convergence to point where~\(f(x_e)\) can be non-zero. 
This leads to numerical issues in practice, since~\(f\) is discontinuous if~\(f(x_e) \neq 0\).
A numerical solver encountering a point~\(x\) that is equal to~\(x_e\) up to machine precision, observes a rapid jump in the vector field~\(f(x)\), which drives away the numerical solution from the intended equilibrium point~\(x_e\).
This is visualised in \cref{fig:numerical-issues}.

In practice, not enforcing~\(g(x_e) = 0\) leads to some training runs aborting due to numerical instability.
If this does not occur, the equilibrium point appears to be shifting during training.
We observed this behaviour in our preliminary experiments, which led us to investigate the causes of this in~\cite{ManekKolter2019}.
In a successful training run, the issues in~\cite{ManekKolter2019} can be masked if the vector field learns to produce~\(f(x_e) \approx 0\).

\begin{figure}
  \centering
  \begin{subfigure}[t]{.49\linewidth}
      \centering
      \includegraphics[width=4cm]{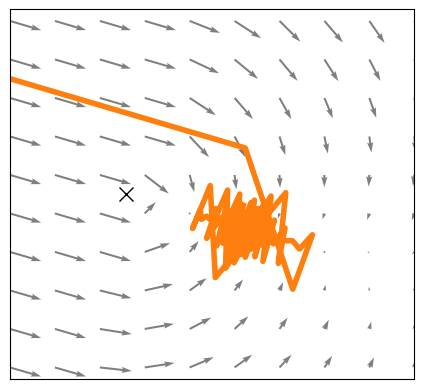}
      \caption{Numerical issues in~\cite{ManekKolter2019}.}
  \end{subfigure}%
  \begin{subfigure}[t]{.49\linewidth}
      \centering
      \includegraphics[width=4cm]{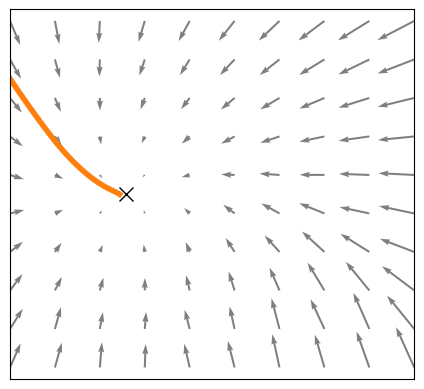}
      \caption{With Our Correction.}
  \end{subfigure}
  \caption{%
    \textbf{Numerical solutions of a vector field stabilised using (a)~\cite{ManekKolter2019} and (b) our approach}.
    In this figure,~\(f(x_e) \neq 0\) when using~\cite{ManekKolter2019} as described in~\cref{sec:main-stability-issue}.
    This leads to numerical instability in~\cite{ManekKolter2019}, which is addressed by our correction.
    In this figure,~\(\manif = \Reals^2\) and~\(V(x) = \distance{\manif}(x, x_e)\).
  }
  \label{fig:numerical-issues}
\end{figure}

\subsection{Training Strategy}\label{sec:nmode-train}
We present an efficient training strategy for fitting \ourframework{}s to a set of motion demonstrations.
Training \cref{eqn:stable-nmode} using gradient descent requires computing a costly second derivative of~\(V\), since it requires computing the gradient of~\(\grad{x}{\manif}V(x)\).
In practice, we found training \cref{eqn:stable-nmode} to require at least twice as much time as training \cref{eqn:base-deriv}, which does not require computing a second derivative.
To train more efficiently, we split the training into three stages, where the first stage has, in turn, two parts:
%
\begin{enumerate}
    \item fit the base vector field~\(g\) to the demonstrations by
    \begin{enumerate}
        \item pretraining on estimated tangents and
        \item fit~\(g\) using multiple shooting~\cite{MultipleShootingNODE};
    \end{enumerate}
    \item train the Lyapunov function~\(V\) to descend along the training trajectories and the solutions of~\(g\); and
    \item fine-tune the assembled stable vector field~\(f\) on the demonstrations.
\end{enumerate}
While the first and second stages find a strong initialisation of~\(g\) and~\(V\), the third stage smoothes out kinks introduced by the stability projection.
In the first stage, we first pretrain on estimated tangent vectors and then leverage multiple shooting~\cite{MultipleShootingNODE} to further speed up training. 
In the following, we define each loss function for a single human demonstration~\(x_{0:N} \in \manif^{N}\) of the motion we want to learn.
In practice, we average each loss function over a small batch of demonstrations.

For the first stage, we estimate the tangent vectors as~\(u_i = \log_{x_i}(x_{i+1})\) for~\(i \in \{0, N-1\}\) and pretrain by minimising the mean squared error (MSE) between~\(g(x_{i})\) and~\(u_{i}\)
\begin{equation*}
    \mathcal{L}_{\text{tangents}} = \frac{1}{N-1}\sum_{i=1}^{N-1} (g(x_{i}) - u_i)^2.
\end{equation*}
Next, multiple shooting divides the demonstration~\(x_{0:N}\) into~\(K\) shorter fragments~\(x_{0:\hat{N}}^{(k)}\) of length~\(\hat{N} \in \{2, \ldots, N\}\) where~\(k \in \{1, \ldots, K\}\) and trains on all fragments in parallel.
Further, it adds a term to the loss function that penalises fragment solutions whose endpoint is not equal to the starting point of the next fragment.
We use a manifold-aware MSE loss together with the multiple shooting penalty
\begin{IEEEeqnarray}{rCl}
    \mathcal{L}_{\text{MS}} &=& 
    \frac{1}{K\hat{N}}\sum_{k=1}^K\sum_{i=0}^{\hat{N}} 
        \distance{\manif}(\approxsol_i^{(k)}, x_i^{(k)})^2 \nonumber\\
    &&{}+ \frac{\lambda_{\text{MS}}}{K-1}\sum_{k=1}^{K-1}\distance{\manif}(\approxsol_{\hat{N}}^{(k)}, \approxsol_{0}^{(k+1)}),\label{eqn:multishoot-loss}
\end{IEEEeqnarray}
where~\(\approxsol_{0:\hat{N}}^{(k)}\) is a numerical solution of the NMODE with vector field~\(g\) for the fragment~\(x_{0:\hat{N}}^{(k)}\) and~\(\lambda_\text{MS} \in \RealsNonNeg\) is a hyperparameter.
During training, we linearly increase the shot length~\(\hat{T}\) within a predefined interval.
Training with longer shots is more costly but necessary to learn high-level features of the trajectories.

For the second stage, we train~\(V\) to decay exponentially along the original training trajectories and the NMODE solutions for~\(g\) that were learned in stage one.
Concretely, we train~\(V\) to minimize the the loss function
\begin{IEEEeqnarray*}{rCl}
    \mathcal{L}_V &=& 
    \frac{1}{N}\sum_{i=1}^{N} \relu{V(\hat{x}_i) - e^{-\alpha t\delta_t}V(x_0)},
\end{IEEEeqnarray*}
where~\(\hat{x}_{0:N}\) is either a human demonstration, or a solution for~\(g\).
In practice, we again average~\(\mathcal{L}_V\) over a small batch of both demonstrations and solutions of~\(g\).

In the third stage, we fine-tune the assembled \ourframework{} vector field~\(f\) using the multiple shooting technique also applied in stage one.
The loss function is \eqref{eqn:multishoot-loss} with the only change that~\(\approxsol_{0:\hat{N}}^{(k)}\) are approximate solutions for~\(f\) instead of~\(g\).


\section{Experiments}\label{sec:experiments}
In this section, we compare our \ourframework{} framework to the existing robotic motion learning frameworks RSDS~\cite{ZhangBeikMohammadiRozo2022} and SDS-RM~\cite{SaverianoAbuDakkaKyrki2023} and demonstrate the practical applicability of our approach to real-world robot tasks.
Additionally, we compare different NMODE solvers from \cref{sec:solve-mode} and compare our three-stage training strategy to training~\eqref{eqn:stable-nmode} directly without a strong initialisation.

\paragraph*{Evaluation Details}
Following previous work~\cite{Lasa,SaverianoAbuDakkaKyrki2023,AuddyEtAl2023Continual}, we train on all trajectories and report training losses.
We use AdamW~\cite{AdamWTrain} for training, the Euler method for solving ODEs, and optimise learning rates, training duration, and shot lengths for multiple shooting by running~\(100\) \texttt{Optuna}~\cite{Optuna} trials for each NMODE.
All hyperparameters are available online\footnote{\url{https://drive.google.com/file/d/1WZlMBKk4kJngwMAxgGIs09xvDWu2HSgU/}} alongside our code.
%
We implement our approach based on \texttt{PyTorch}~\cite{PyTorch} and \texttt{torchdiffeq}~\cite{torchdiffeq}.
Our experiments were run on an Ubuntu 24.04 compute server with an AMD Ryzen Threadripper 3960X CPU.

\subsection{Datasets}\label{sec:experiments-datasets}
\begin{figure}
    \centering
    \begin{tabular}{@{}c@{}c@{}c@{}c@{}}
        \multicolumn{2}{c}{\textbf{Original~\cite{SaverianoAbuDakkaKyrki2023}}} & 
        \multicolumn{2}{c}{\textbf{Improved (Ours)}} \\ \cmidrule(lr){1-2}\cmidrule(lr){3-4}
        \includegraphics[width=.2125\textwidth]{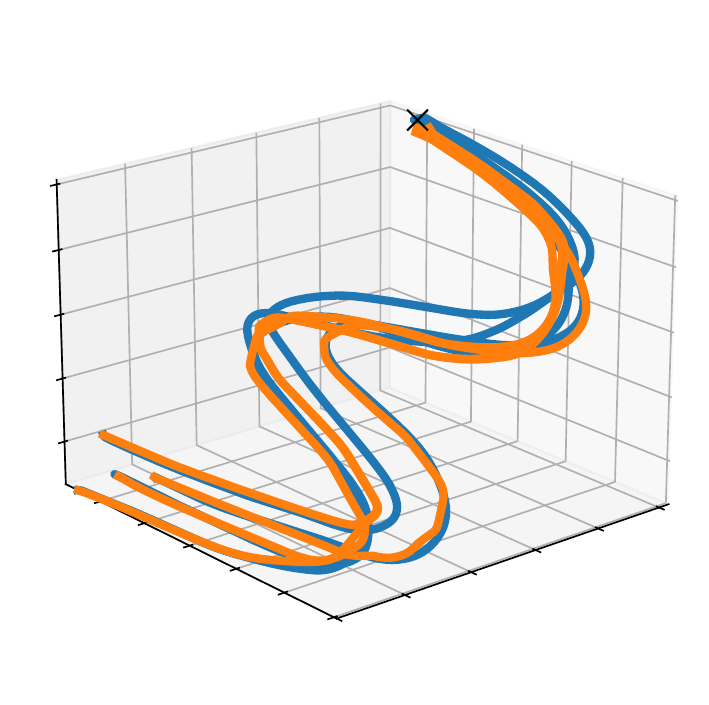} &
        \includegraphics[width=.2125\textwidth]{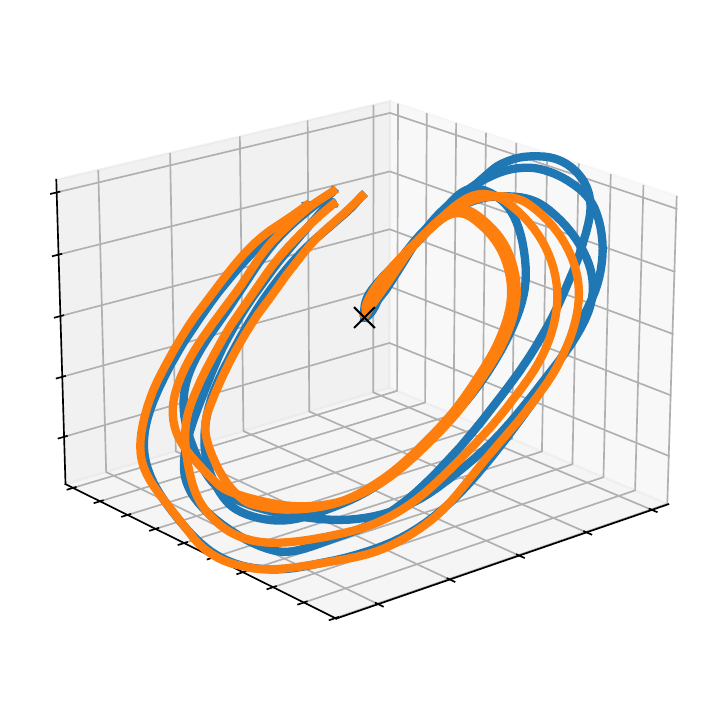} &
        \includegraphics[width=.2125\textwidth]{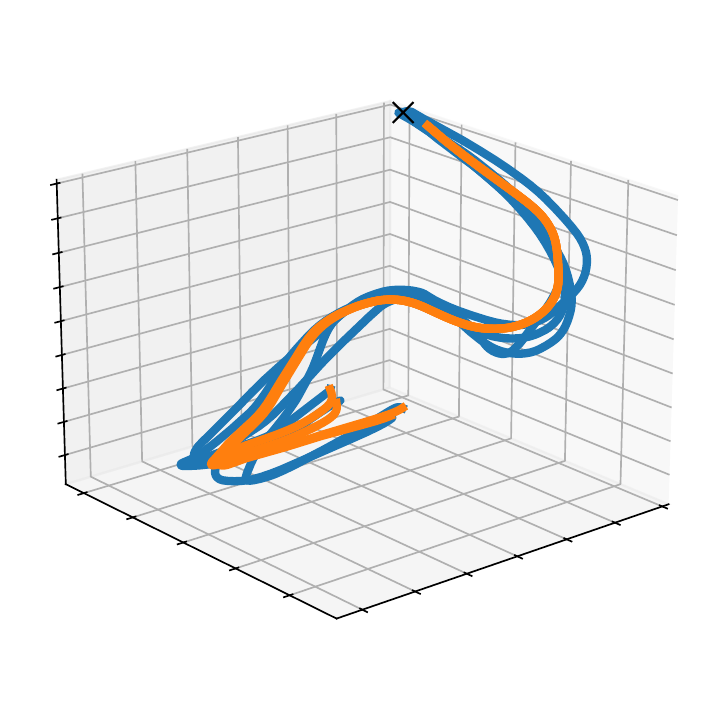} & 
        \includegraphics[width=.2125\textwidth]{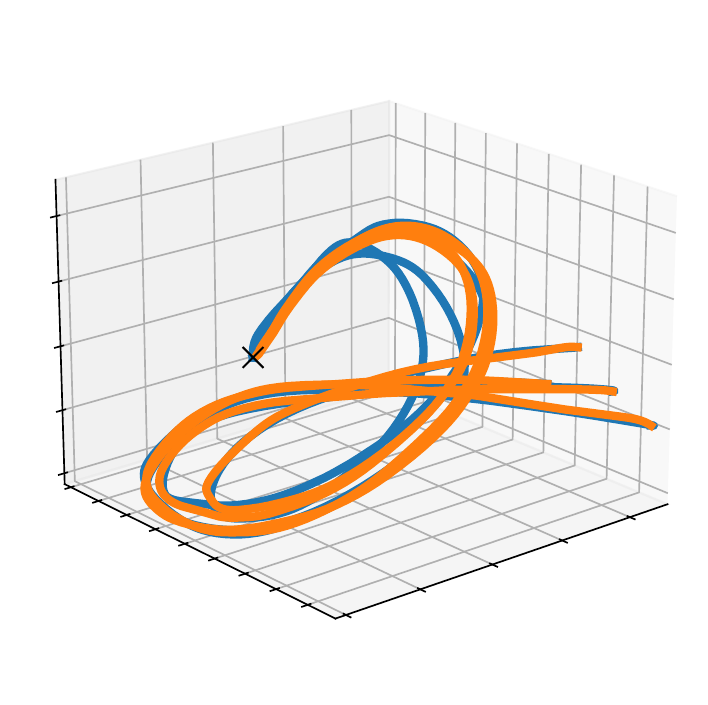} \\
        \texttt{WShape} & \texttt{GShape} & \texttt{WShape} & \texttt{GShape}
    \end{tabular}
    \caption{\textbf{Unit Quaternion LASA Datasets.}
        In our improved dataset, the LASA shapes are visibly distorted by the manifold geometry, while the original dataset maintains the Euclidean geometry.
        The unit quaternions are displayed in an axis-angle representation.
        The stable NMODEs solutions~\legendline{pltorange} follow the demonstrations~\legendline{pltblue} from both datasets.
    }\label{fig:new-riemannian-lasa-results}
\end{figure}

We use the RoboTasks9 dataset~\cite{AuddyEtAl2023Scalable} and create improved versions of the unit quaternion (\(\unitquat\)) and~\(2\times 2\) symmetric positive-definite matrix (\(\spdtwo\)) Riemannian LASA datasets~\cite{SaverianoAbuDakkaKyrki2023}, where~\(\spdtwo\) is equipped with the affine-invariant metric.
RoboTasks9 contains demonstrations for nine robot tasks, such as opening a box or shelving a bottle.
Each task has nine demonstrations of~\(1000\) samples of a robot's end effector pose that lie on the~\(\Reals^3 \times \unitquat\) manifold.

Saveriano et al.~\cite{SaverianoAbuDakkaKyrki2023} construct Riemannian versions of the LASA datase~\cite{Lasa} on~\(\unitquat\) and~\(\spdtwo\) by embedding each 2D LASA trajectory into a fixed tangent space at the goal and then projecting it onto the manifold using the exponential map.
This approach yields trajectories that (\emph{i}) lie in a small neighbourhood of the goal~---~exhibiting minimal curvature~---~and (\emph{ii}) these datasets provide an unrealistic advantage to approaches working in a single tangent space, since such approaches can directly reconstruct the Euclidean data. 
To address these shortcomings, we propose an enhanced transfer procedure. First, as in~\cite{SaverianoAbuDakkaKyrki2023}, we form four~\(\Reals^3\) trajectories for each of the 30 LASA shapes. We estimate velocities by~\(v_i = x_{i+1} - x_i\) for ~\(i \in \{0, \ldots, N-1\}\) and interpolate them linearly to obtain a continuous velocity field~\(v: [0, N] \to \Reals^3\). We then amplify~\(v\) by a factor of~\(20\) to ensure substantial manifold coverage. Finally, we integrate the scaled field~\(d\xi(t)/dt = -20v(t)\) on~\(\unitquat\) and~\(\spdtwo\) starting from a fixed goal point, using the dynamic‐chart method with a fourth‐order Runge–Kutta solver, to generate Riemannian trajectories with significant curvature. The resulting datasets consist of four demonstrations of~\(1000\) samples for each LASA shape, all terminating at a common target. Example trajectories from our~\(\unitquat\) dataset are shown in \Cref{fig:new-riemannian-lasa-results}.

\subsection{Numerical NMODE Solvers}\label{sec:experiments-mode-solvers}
We compare the three numerical NMODE solvers from \cref{sec:solve-mode} on our Riemannian LASA datasets when using our three-stage training strategy.
For this comparison, we select five LASA shapes each from~\(\unitquat\) LASA and~\(\spdtwo\) LASA that we found to be challenging to fit in a preliminary evaluation.
Perhaps surprisingly, using the dynamic chart method (DC) does not outperform the tangent space (TS) and exp-step (Exp) approaches, although DC is considerably more accurate than TS and Exp~\cite{DynamicChartMethod}.

\Cref{tab:comparison-lasa} summarises the results of this experiment.
Our results show that the \ourframework{}s learn to mitigate the inaccuracy of TS and Exp, so that TS and Exp match the performance of DC and even outperform it in many cases. 
This is possible, since TS and Exp use fewer highly non-linear manifold operations than DC, which makes it easier to train NMODEs solved using TS and Exp.
At the same time, TS and Exp allow for faster training.
In our remaining experiments, we use Exp on~\(\unitquat\) and~\(\Reals^3 \times \unitquat\) and TS on~\(\spdtwo\).
\Cref{fig:new-riemannian-lasa-results} visualises a selection of our results on~\(\unitquat\) LASA.

\subsection{Training Strategies}\label{sec:experiments-training-strategies}
We compare our three-stage training strategy from \cref{sec:nmode-train} to training \cref{eqn:stable-nmode} directly from a random parameter initialisation (direct training).
The comparison is based on the five shapes from our~\(\unitquat\) LASA dataset also used in \cref{sec:experiments-mode-solvers}.
We use the Exp-Step method in this experiment.

We first compare the time requirements of training the base NMODE~\eqref{eqn:base-deriv} to training the \ourframework{}~\eqref{eqn:stable-nmode}.
Across~\(100\) repetitions, computing gradients of an NMODE solution with~\(100\) time steps has a median runtime of~\(0.18\)s for \cref{eqn:base-deriv} compared to~\(0.85\)s for \cref{eqn:stable-nmode}\footnote{The inter-quartile ranges (IQR) are within~\(0.004\)s of the median.}.
This difference allows our three-stage training strategy to significantly outpace direct training.
Concretely, our three-stage strategy achives a~\(1.99\)-fold speedup (IQR:~\(1.54\)--\(2.47\)) with a median runtime of~\(135\)s (IQR:~\(131\)s--\(182\)s) compared to~\(364\)s (IQR:~\(207\)s--\(502\)s) for direct training with indepently tuned hyperparameters.
At the same time, both strategies achieve comparable loss values: the median loss of three-stage training is~\(0.11\) (IQR:~\(0.11\)--\(0.18\)), while direct training achieves~\(0.14\) (IQR:~\(0.10\)--\(0.15\)).

\begin{table}
    \centering
    \caption{%
        \textbf{Results for Our Improved Riemannian LASA Datasets}.
        This table reports the root mean squared distance between the human demonstrations and the final learned solutions for different approaches.
        We compare the tangent space (TS), exp-step (Exp), and dynamic chart (DC) methods for numerically solving NMODEs.
    }\label{tab:comparison-lasa}
    \small
    \begin{tabular}{p{1.84cm}@{}ccccc@{\hspace{1em}}ccccc}
        & \multicolumn{5}{c}{\(\unitquat\) (Unit Quaternions)} 
        & \multicolumn{5}{c}{\(\spdtwo\) (Symmetrix Positive-Definite \(2 \times 2\) Matrices)} \\ \cmidrule(lr){2-6}\cmidrule(lr){7-11}
        & \texttt{GShape}\hspace{-2mm} & \texttt{NShape}\hspace{-2mm} & \texttt{WShape} & \hspace{-2.5mm}\texttt{BendedLine}\hspace{-2.5mm} & \texttt{Multi}\hspace{1mm}\texttt{4}
        & \texttt{CShape}\hspace{-2mm} & \texttt{Sshape}\hspace{-2mm} & \texttt{Zshape}\hspace{-2mm} & \texttt{Leaf}\hspace{1mm}\texttt{1}\hspace{-2mm} & \texttt{Multi}\hspace{1mm}\texttt{3} \\ \midrule
        \multicolumn{10}{l}{\textbf{\ourframework{} (Ours)}} \\
        1. TS  & \(\mathbf{0.11}\) & \(0.19\)          & \(\mathbf{0.06}\) & \(0.41\)          & \(0.15\)
               & \(\mathbf{0.0064}\) & \(        0.0056 \) & \(\mathbf{0.0061}\) & \(        0.0047 \) & \(\mathbf{0.0018}\)
        \\
        2. Exp & \(\mathbf{0.11}\) & \(\mathbf{0.18}\) & \(0.10\)          & \(0.33\)          & \(\mathbf{0.11}\)
               & \(0.0068\)          & \(\mathbf{0.0049}\) & \(        0.0079 \) & \(\mathbf{0.0037}\) & \(        0.0020 \)
        \\
        3. DC  & \(0.20\)          & \(0.19\)          & \(0.10\)          & \(\mathbf{0.27}\) & \(0.15\)
               & \(        0.0103 \) & \(        0.0079 \) & \(        0.0090 \) & \(        0.0077 \) & \(        0.0047 \) 
        \\ \midrule
        \textbf{SDS-RM}~\cite{SaverianoAbuDakkaKyrki2023} 
            & \(1.32\)           & \(1.37\)         & \(1.39\)          & \(1.39\)           & \(0.93\)
            & \(0.0146\)         & \(0.0071\)       & \(0.0159\)        & \(0.0091\)         & \(0.0075\)
        \\
    \end{tabular}
\end{table}

\subsection{Comparison with RSDS\texorpdfstring{~\cite{ZhangBeikMohammadiRozo2022}}{}}
We perform a runtime comparison with RSDS~\cite{ZhangBeikMohammadiRozo2022}, demonstrating the computational efficiency of our approach.
Specifically, we compare the runtime of computing a tangent vector of the~\(\Reals^3 \times \unitquat\) manifold using RSDS and~\ourframework{}.
Since no code is available for RSDS, we recreated its model architecture.
Over~\(100\) repetitions, our approach computes tangent vectors in a median runtime of~\(0.002\)s versus~\(0.053\)s and gradients in~\(0.005\)s versus~\(0.161\)s, being one order of magnitude faster than RSDS during deployment and two orders of magnitude faster during training\footnote{All runtime measurements have an IQR within~\(0.0005\)s of the median.}.
Since the training process is discussed only sparsely in~\cite{ZhangBeikMohammadiRozo2022}, we were unable to perform a more detailed comparison with RSDS.

\subsection{Comparison with SDS-RM\texorpdfstring{~\cite{SaverianoAbuDakkaKyrki2023}}{}}
We compare our approach to SDS-RM~\cite{SaverianoAbuDakkaKyrki2023} on our improved Riemannian LASA datasets.
\Cref{tab:comparison-lasa} compares the performance of our \ourframework{}s to the Gaussian mixture models (GMMs) learned by SDS-RM.
For SDS-RM, we manually tune the number of GMM components and report the best result from five random restarts for each dataset.
Across the board, our \ourframework{}s outperform the GMMs learned by SDS-RM.
This demonstrates that our approach is more suitable for learning complex shapes that span Riemannian manifolds than SDS-RM.
Additionally, GMMs are known to scale poorly to high-dimensional manifolds~\cite{abu-dakka2018force}, while~\ourframework{} scales linearly in the manifold dimension, as demonstrated in the next section.

\subsection{Jointly Learning Stable Motions of Several Manipulators}
In this section, we study learning a joint motion of several manipulators based on the RoboTasks9 dataset.
This demonstrates that our approach is suitable for highly complex tasks that require up to nine manipulators and shows that the training time of our approach scales linearly in the manifold dimension.
To obtain a dataset of multi-manipulator demonstrations, we combine tasks from the RoboTasks9 dataset~\cite{AuddyEtAl2023Scalable} in a fixed order.
\Cref{fig:manytasks} presents our results on this dataset, demonstrating that our approach can effectively learn up to nine tasks while the training time scales linearly in the manifold dimension.

\begin{figure*}
    \centering
    \begin{tabular}{l@{\hspace{3mm}}c@{\hspace{4mm}}c@{\hspace{4mm}}c@{\hspace{4mm}}c@{\hspace{4mm}}c@{\hspace{4mm}}c@{\hspace{4mm}}c@{\hspace{4mm}}c}
           \includegraphics[width=1.45cm]{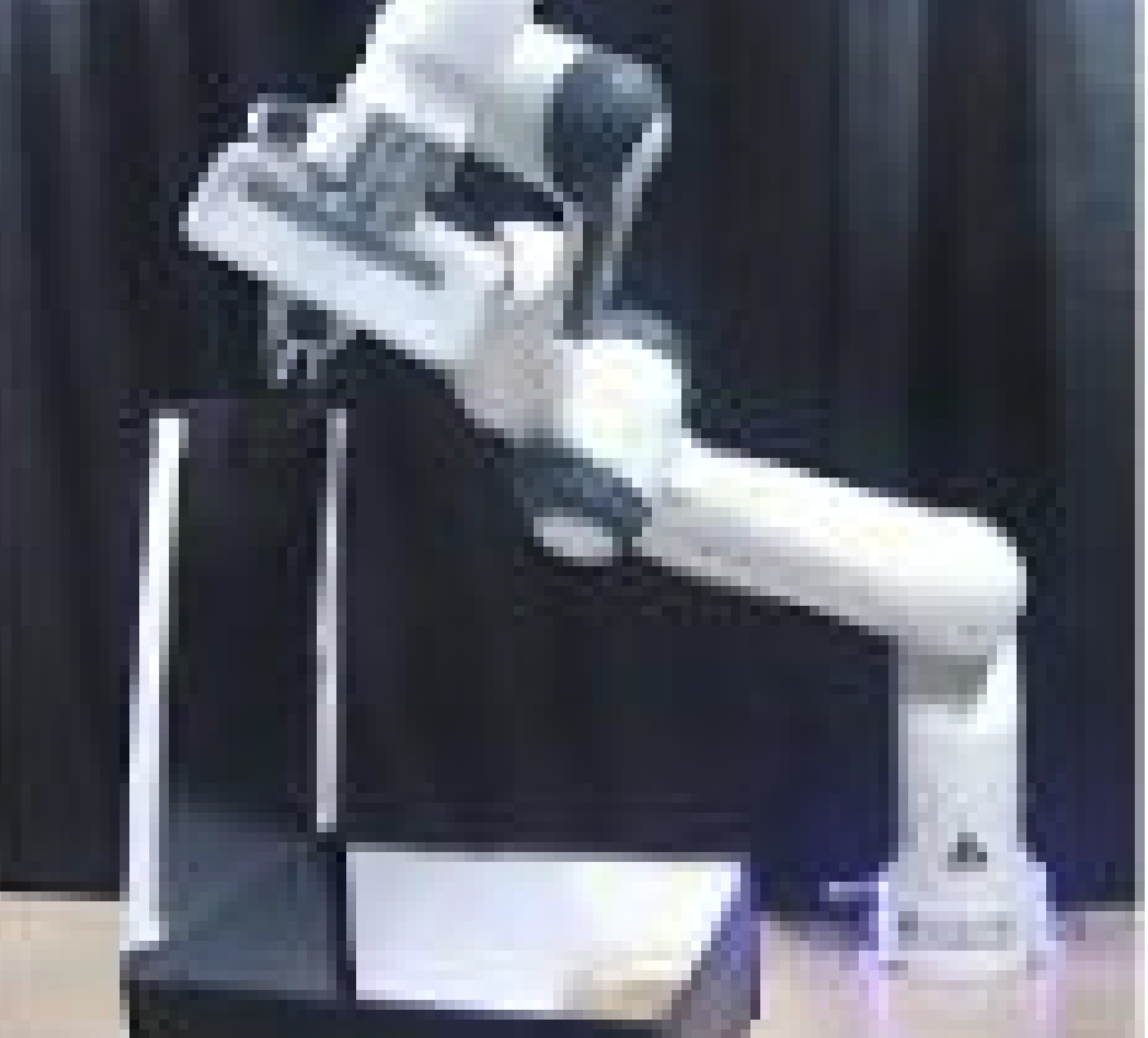}               \clap{\hspace*{3.5mm}\raisebox{.45cm}{\(\times\)}}
         & \includegraphics[width=1.45cm]{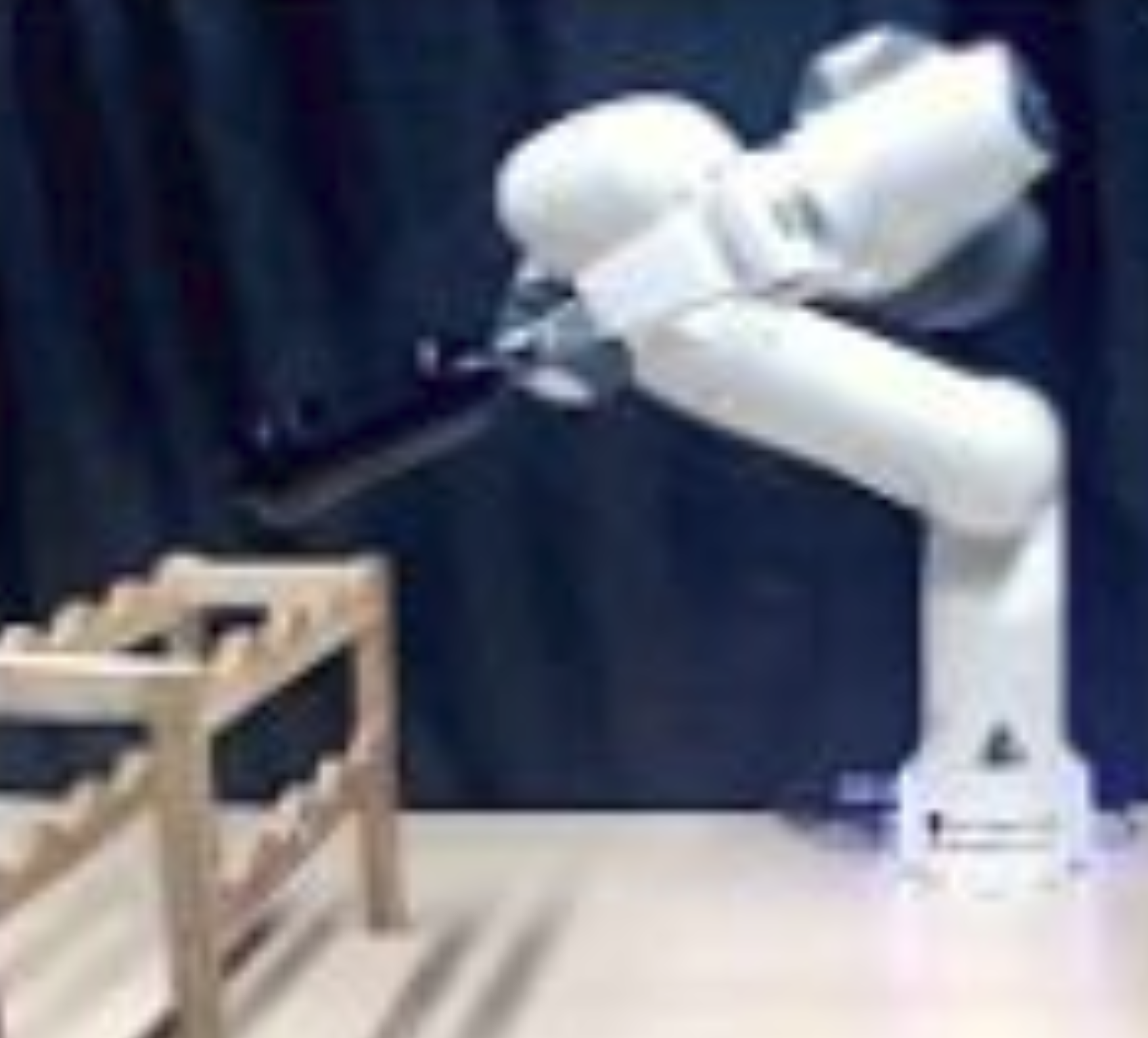}          \clap{\hspace*{3mm}\raisebox{.45cm}{\(\times\)}}
         & \includegraphics[width=1.45cm]{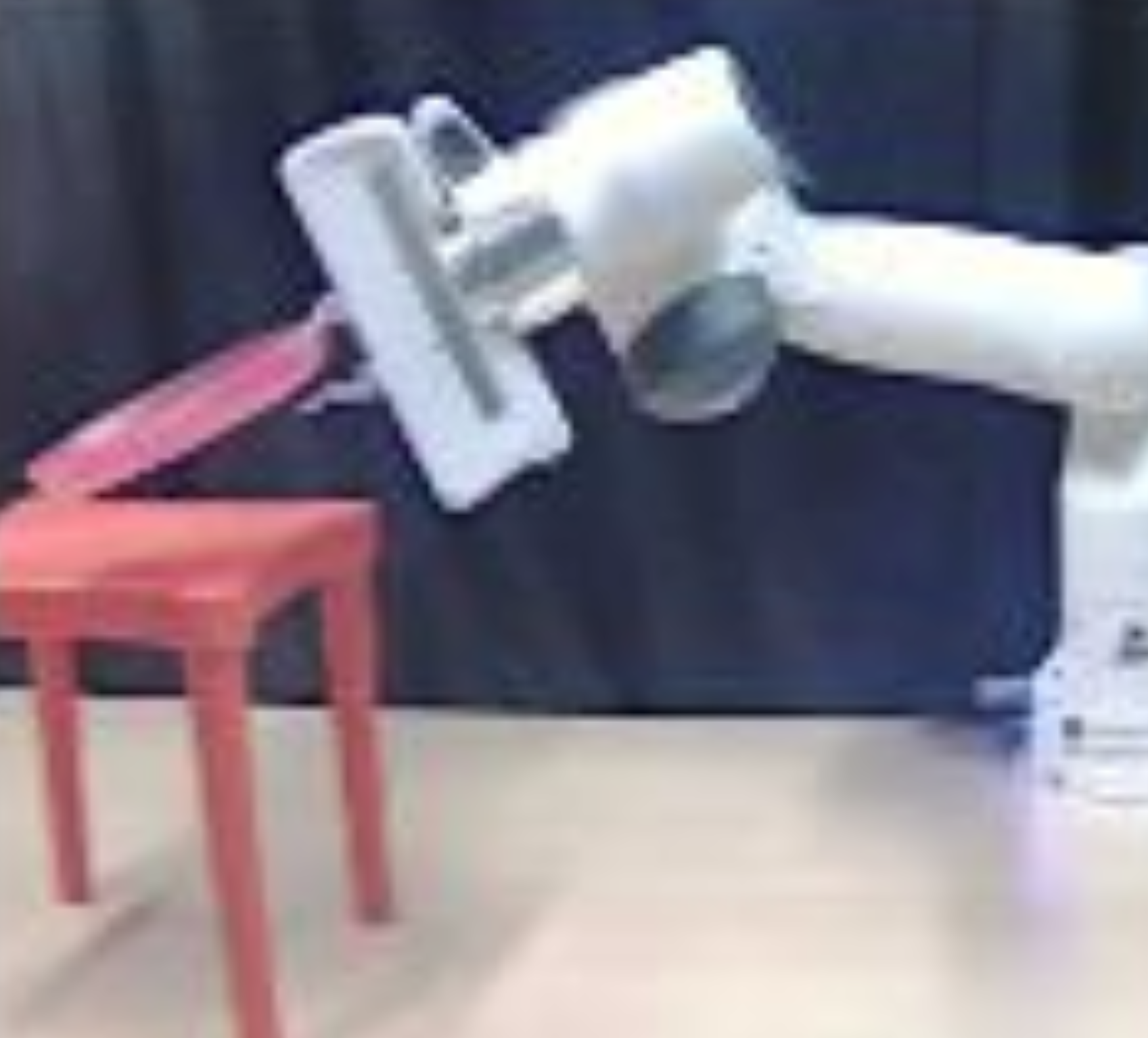}  \clap{\hspace*{3mm}\raisebox{.45cm}{\(\times\)}}
         & \includegraphics[width=1.45cm]{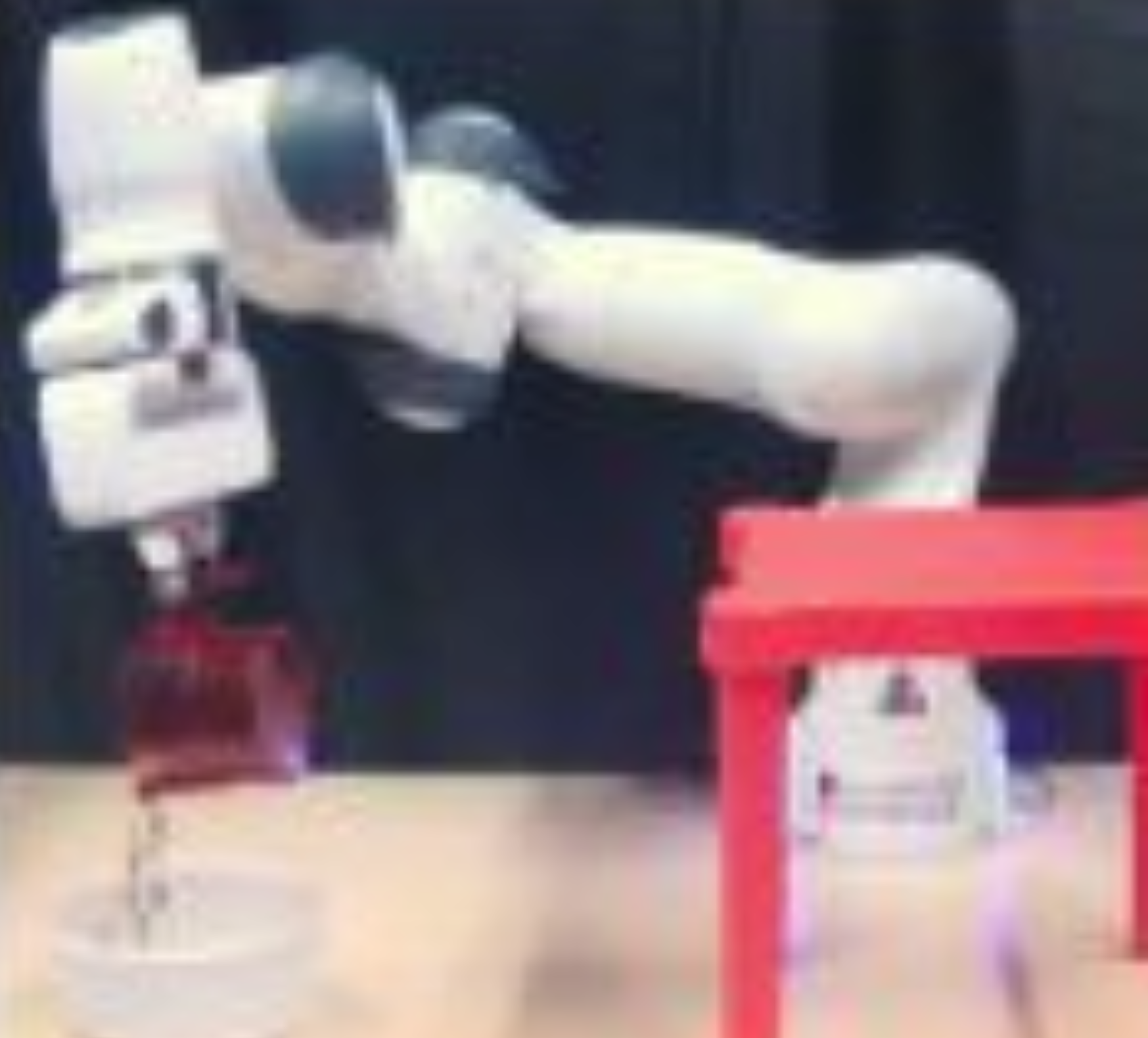}               \clap{\hspace*{3mm}\raisebox{.45cm}{\(\times\)}}
         & \includegraphics[width=1.45cm]{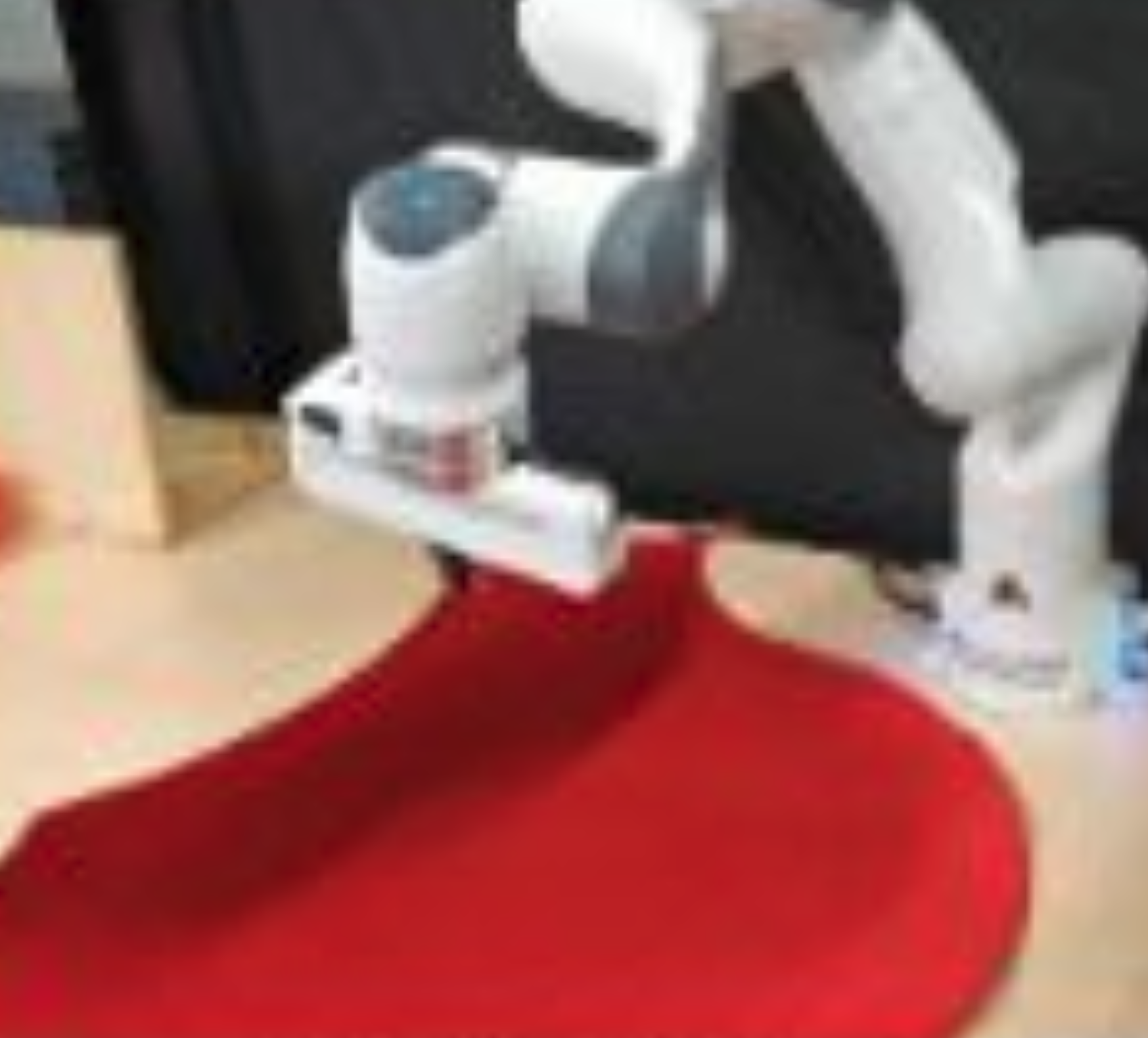}            \clap{\hspace*{3mm}\raisebox{.45cm}{\(\times\)}}
         & \includegraphics[width=1.45cm]{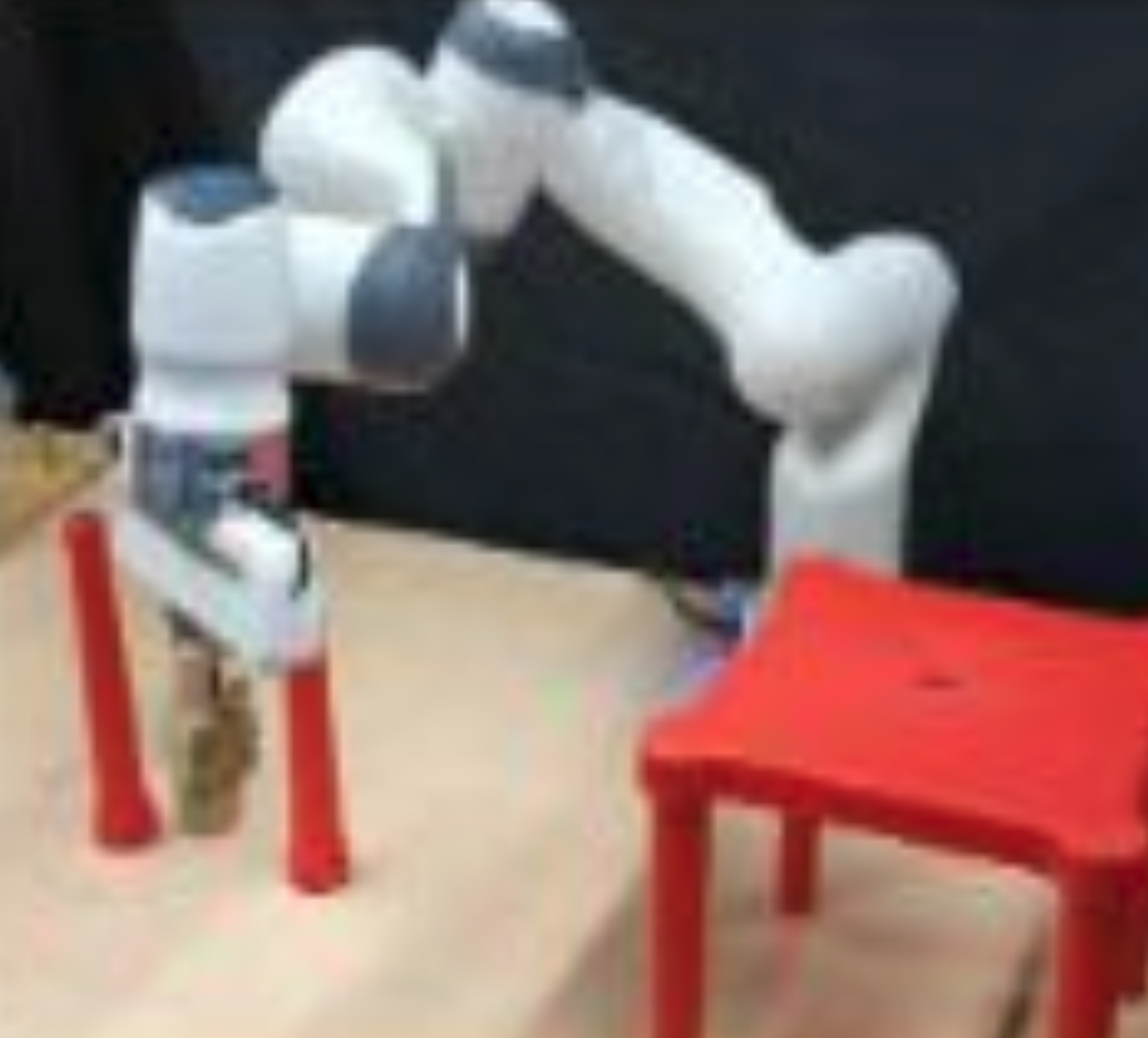}        \clap{\hspace*{3mm}\raisebox{.45cm}{\(\times\)}}
         & \includegraphics[width=1.45cm]{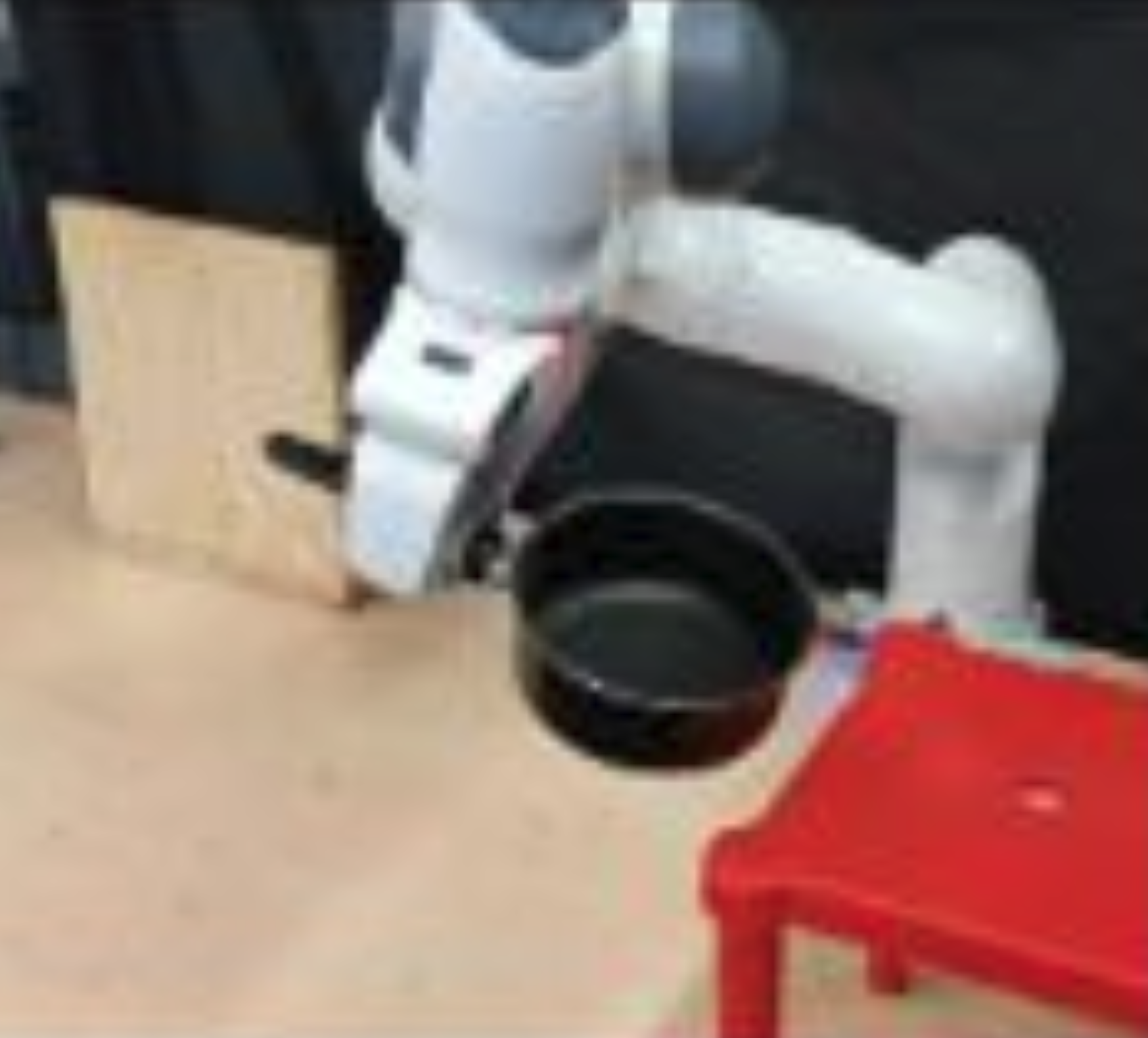}             \clap{\hspace*{3mm}\raisebox{.45cm}{\(\times\)}}
         & \includegraphics[width=1.45cm]{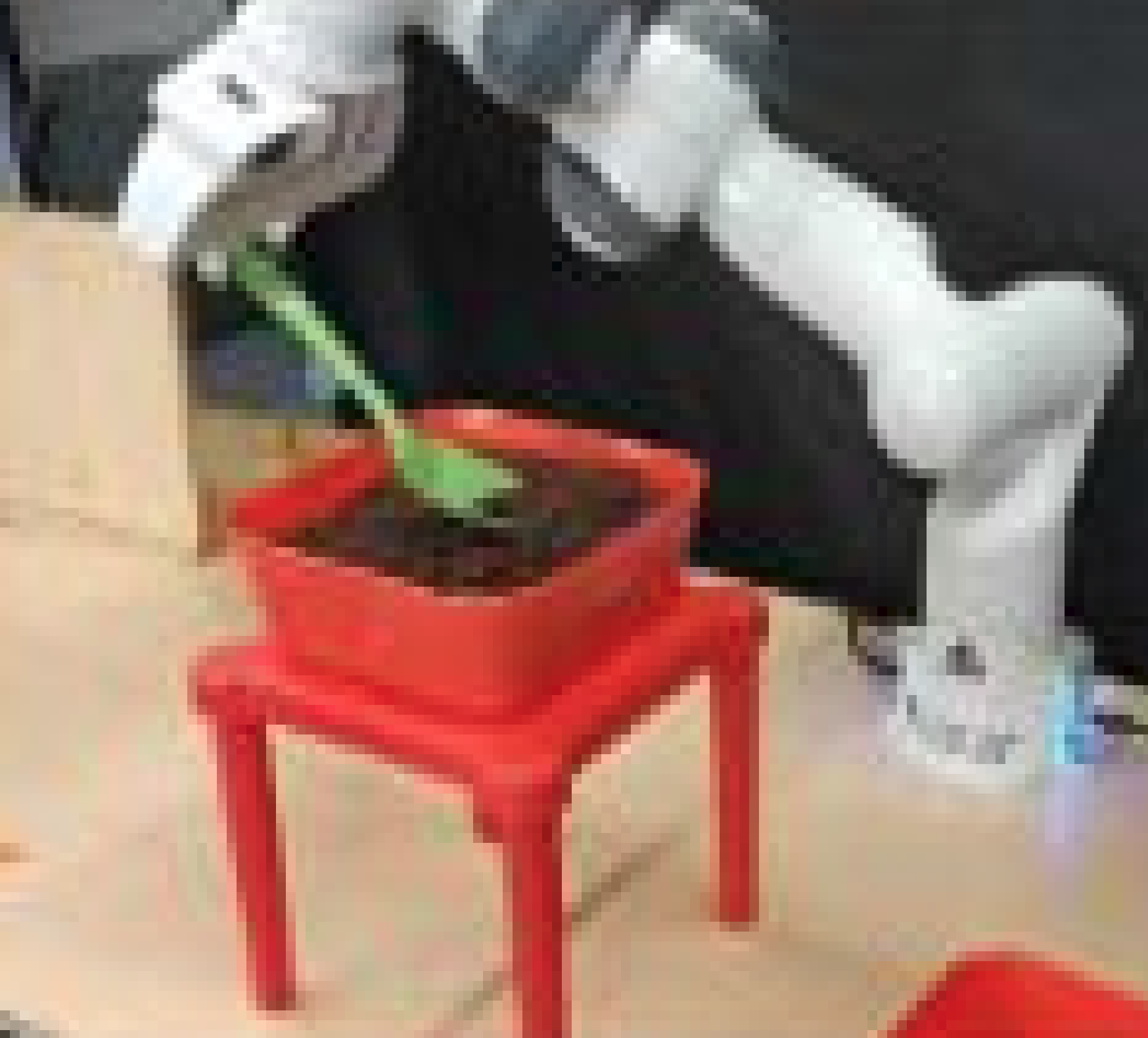}              \clap{\hspace*{3mm}\raisebox{.45cm}{\(\times\)}}
         & \includegraphics[width=1.45cm]{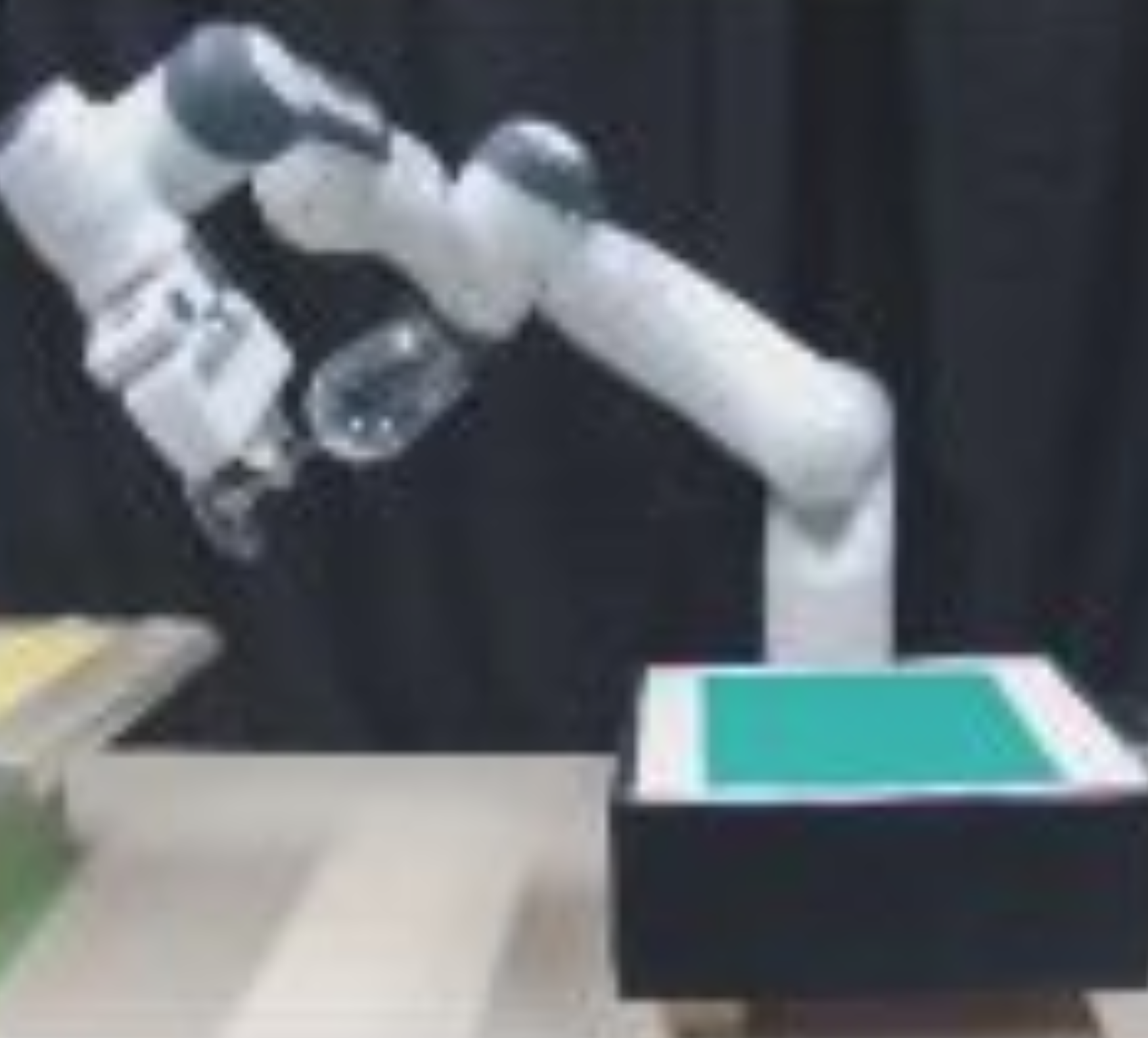}
         \\[-.45em]
         \scriptsize Open Box & \scriptsize Shelve Bottle & \scriptsize Stack Plate & \scriptsize Pouring & \scriptsize Fold Mat & \scriptsize Navigate & \scriptsize Pan on Stove & \scriptsize Scooping & \scriptsize Glass Upright \\ 
         \textbf{\#Tasks} & \textbf{2} & \textbf{3} & \textbf{4} & \textbf{5} & \textbf{6} & \textbf{7} & \textbf{8} & \textbf{9} \\ 
         \midrule
         \textbf{Loss}       &  \(0.05\) &  \(0.12\) &  \(0.19\) &  \(0.21\) &  \(0.31\) &  \(0.29\) &  \(0.26\) &  \(0.26\) \\
         \textbf{Train Time} & \(1429\)s & \(1636\)s & \(1840\)s & \(1998\)s & \(2213\)s & \(2416\)s & \(2612\)s & \(2794\)s \\ \midrule
         \textbf{Dimension}
         & \(12\) & \(18\) & \(24\) & \(30\) & \(36\) & \(42\) & \(48\) & \(54\) \\
    \end{tabular}
    \caption{%
        \textbf{Multi-Manipulator Results}.
        We train stable NMODEs to jointly perform~\(2\)--\(9\) tasks controlling multiple manipulators simultaneously.
        \enquote{Train Time} is the duration of training, \enquote{Loss} is the root mean squared distance between the demonstrations and the learned solutions, and \enquote{Dimension} is the dimension of the manifold.
        While the loss is independent of problem size, the runtime of our approach scales linearly with the problem size.
        The task illustrations are taken from~\cite{AuddyEtAl2023Continual}.
    }\label{fig:manytasks}
\end{figure*}


\subsection{Real-World Robot Experiment}
\begin{figure}
    \centering
    \includegraphics[width=\linewidth]{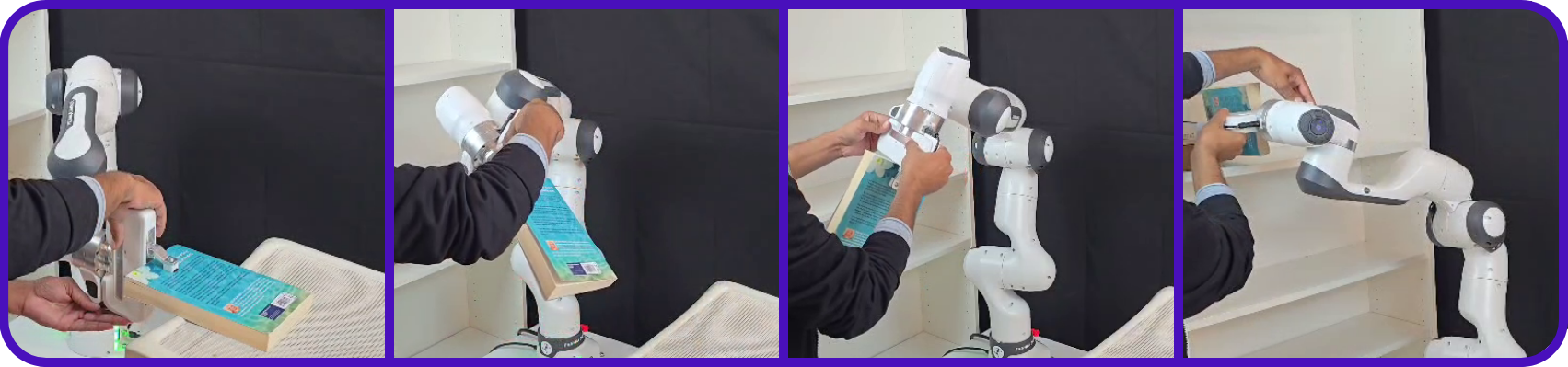}
    \put(-215,40){\color{red}\textbf{Start}}
    \put(-214,30){\color{red}\textbf{Pose}}
    \put(-48,16){\color{red!85!black}\textbf{Goal}}
    \put(-47,6){\color{red!85!black}\textbf{Pose}}
    \\[.05cm]
    \includegraphics[width=\linewidth]{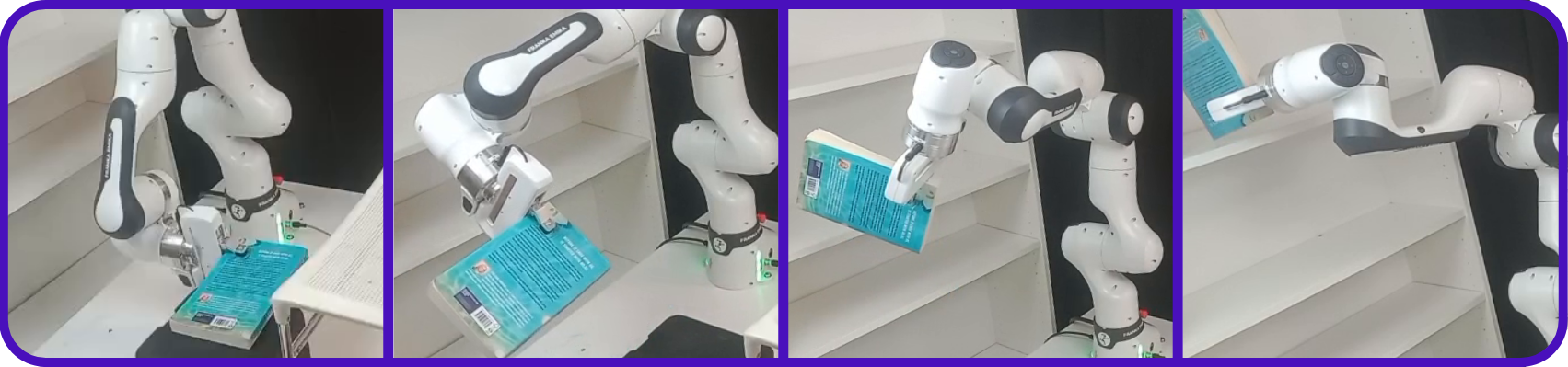}
    \caption{
        \textbf{Book Shelving Task}. 
        \emph{Top Row}: A librarian teaches the robot how to place a lying book upright on a shelf.
        \emph{Bottom Row}: The robot autonomously performs the book shelving motion using \ourframework{}, successfully executing the task from starting poses not contained in the demonstrations.
    }\label{fig:robot-experiment}
\end{figure}

Finally, we demonstrate the real-world applicability of our approach by learning a practical, real-world motion and deploying our learned \ourframework{} on a Franka Emika Panda robot.
Many future robot applications in households, shops, or libraries require placing objects on shelves.
We choose the library task of placing a lying book upright on a shelf, as illustrated in \cref{fig:robot-experiment}, for this experiment.
This task requires manipulating both the position and the orientation of the book, meaning that our demonstrations evolve on~\(\Reals^3 \times \unitquat\).
We collect~\(20\) human demonstrations and learn a \ourframework{} to follow these demonstrations.
\Cref{fig:robot-experiment} shows that the robot is able to perform the motion correctly, also from starting points that were not contained in the human demonstrations.
A video of this robot experiment is available in the supplementary material.


\section{Conclusion}
We propose \ourframework{}, a general and expressive framework based on neural manifold ODEs (NMODEs) for teaching robots complex motions from demonstrations.
By leveraging NMODEs, we enable robots to track data modalities that evolve on Riemannian manifolds, such as orientation.
Our theoretical analysis proves that \ourframework{}s are stable, thereby guaranteeing predictable, safe, and robust robot motions.
We provide an efficient training strategy for applying \ourframework{} in practice.
Our experimental results establish that \ourframework{} outperforms existing frameworks for learning stable dynamical systems on manifolds and demonstrate the practical applicability of our approach.
\bibliographystyle{IEEEtran}
\bibliography{IEEEabrv,main}


\end{document}